\documentclass[12pt,a4paper]{article}
\usepackage[a4paper,margin=3cm]{geometry}
\usepackage{amsmath,amsfonts,amsthm,amssymb,dsfont,lmodern,mathrsfs}
\usepackage{xfrac}
\usepackage{natbib}
\usepackage[utf8]{inputenc}
\usepackage[T1]{fontenc}
\usepackage[shortlabels]{enumitem}
\usepackage[dvipsnames]{xcolor}
\usepackage[colorlinks,linkcolor=blue,citecolor=blue,urlcolor=blue]{hyperref}
\usepackage{doi}
\usepackage{tikz}
\usepackage[margin=1cm]{caption}
\usepackage{cleveref}

\usepackage{parskip}
\begingroup
\makeatletter
\@for\theoremstyle:=definition,remark,plain\do{%
  \expandafter\g@addto@macro\csname th@\theoremstyle\endcsname{%
    \addtolength\thm@preskip\parskip }%
}
\endgroup

\usepackage{titlesec}
\titleformat{\paragraph}[hang]{\bfseries}{}{0mm}{}
\titlespacing{\paragraph}{0mm}{\baselineskip}{0.5ex}

\theoremstyle{plain}
\newtheorem{theorem}{Theorem}[section]
\newtheorem{lemma}[theorem]{Lemma}

\newtheorem{proposition}[theorem]{Proposition}
\newtheorem{conjecture}[theorem]{Conjecture}
\newtheorem{definition}[theorem]{Definition}
\newtheorem{assumption}[theorem]{Assumption}
\Crefname{assumption}{Assumption}{Assumptions}
\theoremstyle{definition}
\newtheorem{example}[theorem]{Example}
\newtheorem{remark}[theorem]{Remark}
\newtheorem{algorithm}[theorem]{Algorithm}

\newcommand{\rmd}{\mathrm{d}}
\newcommand{\1}{\mathds{1}}
\renewcommand{\P}{\mathrm{P}}
\newcommand{\E}{\mathbb{E}}
\newcommand{\norm}[2][]{#1\lVert #2 #1\rVert}
\newcommand{\abs}[2][]{#1\lvert #2 #1\rvert}

\newcommand{\symdiff}{\mathbin{\bigtriangleup}}
\DeclareMathOperator*{\argmin}{arg\,min}

\newcommand{\pderiv}[3][{}]{\frac{\partial^{#1}#3}{\partial{#2}^{#1}}}
\newcommand{\deriv}[2][{}]{\frac{\rmd^{#1}}{\rmd{#2}^{#1}}}

\newcommand{\leb}{\mathrm{Leb}}

\renewcommand{\epsilon}{\varepsilon}
\renewcommand{\rho}{\varrho}
\renewcommand{\phi}{\varphi}
\renewcommand{\emptyset}{\varnothing}

\begin{document}

\title{A large-sample theory for infinitesimal gradient boosting}
\author{Clément Dombry\footnote{Universit{\'e} de Franche-Comt{\'e}, CNRS, LmB, F-25000 Besan{\c c}on, France.\\Email: \url{clement.dombry@univ-fcomte.fr},\ \url{jean-jil.duchamps@univ-fcomte.fr}} \ and Jean-Jil Duchamps\footnotemark[1]}
\maketitle

\begin{abstract}
Infinitesimal gradient boosting \citep{DD21} is defined as the vanishing-learning-rate limit of the popular tree-based gradient boosting algorithm from machine learning. It is characterized as the solution of a nonlinear ordinary differential equation in a infinite-dimensional function space where the infinitesimal boosting operator driving the dynamics depends on the training sample. We consider the asymptotic behavior of the model in the large sample limit and prove its convergence to a deterministic process. This population limit is again characterized by a differential equation that depends on the population distribution. We explore some properties of this population limit: we prove that the dynamics makes the test error decrease and we consider its long time behavior.
\end{abstract}

\textbf{Keywords:} gradient boosting, large sample theory, softmax gradient tree.

\textbf{MSC 2020 subject classifications:} primary 62G05; secondary 60F15, 60J25.

\tableofcontents

\section{Introduction}
Tree-based gradient boosting \citep{F01} is one of the most successful algorithm from machine learning. It provides a powerful and versatile methodology in supervised learning and achieves excellent performance in prediction problems where one aims at understanding the relationship between a response variable (target) and explanatory variables (features). Its modern implementation in XGBoost \citep{CG16} is involved in countless applications. Several theoretical and statistical works have been devoted to the understanding of the good performance of boosting, see e.g. \cite{J04}, \cite{LV04}, \cite{BLV04}, \cite{ZY05} to cite only a few. The primary focus is to establish the consistency of the method, meaning that near optimal error rates can be achieved provided sufficiently large training data is available.

 On the other hand, theoretical results considering the time dynamics of gradient boosting are relatively scarce. A recent advance in that direction is the model of infinitesimal gradient boosting \citep{DD21} that provides a mathematical characterization of the vanishing-learning-rate limit of tree-based gradient boosting by a nonlinear ordinary differential equation in an infinite-dimensional function space.
More precisely, Theorem~1.9 in \citep{DD21} proves the uniform convergence on compact sets of the rescaled gradient boosting process to a deterministic limit, called infinitesimal gradient boosting, as the learning rate $\lambda$ tends to $0$ and Theorem~1.10 states that the rate of convergence is of order $\sqrt{\lambda}$.  
Infinitesimal gradient boosting can thus be seen as an approximation of gradient boosting  in the small learning rate limit. It has the main advantage to make the dynamics deterministic (for fixed input), i.e.\ the randomness of the stochastic algorithm vanishes. The limit dynamics is then characterized by the infinitesimal boosting operator and it ensures that the training error is non-increasing in time. The main purpose of this paper is to focus on the dependency of the infinitesimal gradient boosting with respect to the input sample and to analyze its large-sample asymptotics.

This approach is very much related to the gradient flow approximation of stochastic gradient descent \citep{DDB20}, with the additional difficulty that the dynamics takes place in a function space and that the structure of tree-functions has to be handled. For more details on infinitesimal gradient boosting and its positioning with respect to the gradient boosting literature, the reader should refer to \cite{DD21}. Other works also embed the discrete dynamics of stochastic algorithms  --- specifically AdaBoost~\citep{LMV15} and gradient descent~\citep{Lat21} --- into continuous-time processes. However, in these articles the infinitesimal-time-step limit is not taken, instead the authors rely on piecewise-deterministic Markov processes to account for the discrete steps of the algorithm.

The purpose of this paper is to analyze the large sample theory of infinitesimal gradient boosting and we prove that a deterministic population limit exists. More precisely, the infinitesimal boosting operator driving the dynamics converges as the sample size goes to infinity, implying the convergence of the solutions of the corresponding ODEs. Once convergence is established, we study some properties of the limit and we prove in particular that the population dynamics ensures that the test error is non-increasing. Furthermore, we explore the long-time properties of population boosting. We expect and conjecture that as times goes to infinity, the test error converges to its minimum and  the boosting predictor to the Bayes predictor. Unfortunately, proving such results turns out to be surprisingly difficult and out of reach for the moment, in spite of substantial efforts, so that we provide only partial results in this direction.

The structure of the paper is as follows. Section~\ref{sec:main-results} is devoted to the presentation of the model, assumptions and  results. We first recall the setting of infinitesimal gradient boosting and our main assumptions;  then we state our  results regarding the convergence of infinitesimal gradient boosting when the sample size goes to infinity (Theorem~\ref{thm:asymptotic_igb}) ; finally we describe  some important properties of the population limit. Notions and preliminary results that play an important role in our analysis are introduced in Section~\ref{sec:preliminaries}.
The proofs of the most important results are postponed to Section~\ref{sec:proofs-main-results}, while some of the proofs of
Section~\ref{sec:preliminaries} can be found in Appendix~\ref{sec:proofs-preliminaries}.

\section{Setting and main results}\label{sec:main-results}

\subsection{Setting and notation}\label{sec:setting}
We introduce the setting of infinitesimal gradient boosting developed by \cite{DD21} and try to provide a short yet self-contained presentation. Further details can be found in \cite{DD21}.

\textbf{Supervised statistical learning framework}. We observe a response variable, or target, $Y$ with state space $\mathcal{Y}\subset \mathbb{R}$ jointly with a vector of covariates, or features, $X$ taking values in $[0,1]^p$. We want to construct a model to predict the target $Y$ in view of the features $X$.

 We denote by $\mathrm{P}$ the joint distribution of  $(X,Y)$ and let $(X_i,Y_i)_{i\geq 1}$ be independent copies of $(X,Y)$.
We will also denote by $\P_X$ the marginal distribution of $X$.
 A predictor is a measurable function $F:[0,1]^p\to\mathbb{R}$ used to predict $Y$ in view of $X$. A loss function $L:\mathcal{Y}\times \mathbb{R}\to\mathbb{R}$ compares the observation $y$ and its prediction $F(x)$; the loss $L(y, F(x))$ is interpreted as a prediction error that we want to minimize.  The Bayes risk is defined by $\inf_F \mathbb{E}[L(Y,F(X))]$, i.e.\ the infimum  expected risk over all possible predictors. A predictor $F^*$ achieving the Bayes risk is called a Bayes predictor. The task is to build a predictor $\hat F_n$ using only the first $n$ observations $(X_i,Y_i)_{1\leq i\leq n}$ as a training set and that approaches the Bayes risk as the size of the training sample grows, that is we want $\mathbb{E}[L(Y,\hat F_n(X))]\to \mathbb{E}[L(Y,F^*(X))]$ as $n\to\infty$.
 
 \begin{example}\label{examples} Throughout the paper, we will illustrate our results and assumptions on the following three cases:
 \begin{itemize}
 \item \textbf{regression with squared error} where $\mathcal{Y}=\mathbb{R}$ and $L(y,z)=\frac{1}{2}(y-z)^2$;  assuming $\mathbb{E}[Y^2]<\infty$, the Bayes predictor is the regression function $F^*(x)=\mathbb{E}[Y\mid X=x]$;
 \item \textbf{binary classification with cross-entropy} where $\mathcal{Y}=\{0,1\}$ and $L(y,z)=-yz+\log(1+e^z)$; letting $p(x)=\mathbb{P}(Y=1\mid X=x)$ be the success probability, the Bayes predictor is equal to its logit $F^*(x)=\log(p(x)/(1-p(x))$;
 \item \textbf{binary classification with exponential loss} where $\mathcal{Y}=\{-1,1\}$ and  $L(y,z)=e^{-yz}$; the Bayes predictor is $F^*(x)=\frac{1}{2}\log(p(x)/(1-p(x))$ with $p(x)=\mathbb{P}(Y=1\mid X=x)$.
  \end{itemize}
 \end{example}

\textbf{Softmax gradient trees}. \cite{F01} introduces gradient boosting as an additive model that sequentially learns a sequence of trees in order to minimize the training error. The procedure is akin to gradient descent and at each step, a gradient tree is fitted and added  to the current model; a shrinking factor called learning rate is introduced that plays the same role as the step size in gradient descent. We detail in the following the construction of gradient trees. For more details on the gradient boosting algorithm, we refer to \cite{F01} and \citet[Chapter~10]{ESL}. 

Given a predictor $F:[0,1]^p\to\mathbb{R}$, the gradient tree is obtained by fitting a (randomized) regression tree to the residuals and performing a line search approximation in the different leaves. More precisely, the residuals of a predictor $F$ are defined as
\[
r_i=-\pderiv{z}{L}(y_i,F(x_i)),\quad 1\leq i\leq n.
\]
Note that Friedman used the term \emph{pseudo-residual} in his seminal paper; here for simplicity we will use the term \emph{residual} to denote the $(r_i)_{1\leq i\leq n}$, even in cases other than the regression case --- in which case we indeed have standard residuals $r_i=y_i-F(x_i)$.
A (randomized) regression tree with depth $d\geq 1$ is fitted on $(x_i,r_i)_{1\leq i\leq n}$, yielding a partition $(A_v)_{v\in\{0,1\}^d}$ of $[0,1]^p$ into hyperrectangles called leaves. In this notation,  $\{0,1\}^d$ corresponds to the $2^d$ terminal nodes of a binary tree with depth $d$, see Section~\ref{sec:preliminaries-algo} for more details.  The line search in leaf $A$ consists in searching for the best additive update, i.e.
\[
\argmin_{h\in\mathbb{R}} \sum_{i=1}^n L(y_i,F(x_i)+h)\1_{A}(x_i).
\]
Its one-step approximation  performs a single Newton-Raphson step, yielding the explicit update
\begin{equation}\label{eq:leaf-values}
\tilde r(A)=-\frac{\sum_{i=1}^n \pderiv{z}{L}(y_i,F(x_i))\1_{A}(x_i)}{\sum_{i=1}^n \pderiv[2]{z}{L}(y_i,F(x_i))\1_{A}(x_i)},
\end{equation}
with the convention $0/0=0$. The gradient tree finally writes
\begin{equation}\label{eq:gradient-tree}
T(x)=\sum_{v\in\{0,1\}^d}  \tilde{r}(A_v)\1_{A_v}(x),\quad x\in[0,1]^p.
\end{equation}
We next provide some details on the construction of the partition $(A_v)_{v\in\{0,1\}^d}$ associated with (randomized) regression tree. Starting with the trivial partition $A=[0,1]^d$ into a single leaf (depth $0$), binary splitting is applied recursively with depth $d$ so as to obtain a partition into $2^d$ leaves. Binary splitting selects a covariate $j\in\{1,\ldots,p\}$ and a threshold $u\in [0,1]$ and then divides the leaf $A=\prod_{l=1}^p [a_l,b_l]$ into 
\begin{align}
A_{0}&=A\cap \{x:x_j\leq (1-u)a_j+ub_j\},\nonumber\\
A_{1}&=A\cap\{x:x_j>(1-u)a_j+ub_j\}.\label{eq:split}
\end{align}
Different splitting rule may be used and are generally defined via a score measuring the heterogeneity between the two leaves. For regression trees, the usual score for the split is the intergroup variance
\begin{align*}
  \widetilde{\Delta} &= \frac{n(A_{0})}{n}(\bar r(A_{0})-\bar r(A))^2+\frac{n(A_{1})}{n}(\bar r(A_{1})-\bar r(A))^2\\
  & = \frac{n(A_{0})}{n}\bar r(A_{0})^2 + \frac{n(A_{1})}{n}\bar r(A_{1})^2 - \frac{n(A)}{n}\bar r(A)^2,
\end{align*}
where $n(A)$ and $\bar r(A)$ denote respectively the number of observations and mean residual in leaf $A$, and similarly for $A_{0}$, $A_{1}$.
The last term can be ignored since it depends only on the original leaf $A$, not on the specific split, so that the score we consider is
\begin{equation}\label{eq:score}
  \Delta = \frac{n(A_{0})}{n}\bar r(A_{0})^2 + \frac{n(A_{1})}{n}\bar r(A_{1})^2.
\end{equation}
In its original version \citep{BFOS84}, the algorithm uses greedy binary splitting, meaning that the covariate $j$ and threshold $u$ that are selected maximize the score $\Delta$. Another possibility, explored by Extra-Trees \citep{GEW06}, is to restrict the search of the best split within a subset of $K$ randomly chosen proposals $(j_k,u_k)_{1\leq k\leq K}$. The proposals are independent and uniform on $\{1,\ldots,p\}\times [0,1]$.  When $K=1$, the split $(j,u)$ is chosen completely at random, whence the name \emph{completely random trees}. Softmax regression trees \citep{DD21}  were proposed for the purpose of regularization of the strong argmax in Extra-Trees. Given  $K$ random proposals $(j_k,u_k)_{1\leq k\leq K}$, the scores $(\Delta_k)_{1\leq k\leq K}$ corresponding to the different proposals are computed and the threshold $(j,u)$ is randomly chosen according to the softmax distribution
\begin{equation}\label{eq:softmax}
\mathbb{P}((j,u)=(j_k,u_k))=\frac{e^{\beta \Delta_k}}{\sum_{l=1}^Ke^{\beta \Delta_l}},\quad 1\leq k\leq K.
\end{equation}
The parameter $\beta\geq 0$ allows to interpolate between completely random trees ($\beta=0$) and Extra-Trees ($\beta=\infty$). 

In order to ease the asymptotic analysis, it is useful to see the procedure  as a function of the empirical distribution $\P_n=\frac{1}{n}\sum_{i=1}^n \delta_{(X_i,Y_i)}$ associated to the sample. We use the short notation $\P_n[G(x,y)]=\int G(x,y)\P_n(\rmd x\rmd y)$ to denote the integral of a function $G:\mathbb{R}\times\mathcal{Y}\to\mathbb{R}$ with respect to $\P_n$. The leaf values defined in Equation~\eqref{eq:leaf-values} can be rewritten as
\begin{equation}\label{eq:leaf-values2}
\tilde r(A)=-\frac{\P_n[\pderiv{z}{L}(y,F(x))\1_{A}(x)]}{\P_n[ \pderiv[2]{z}{L}(y,F(x))\1_{A}(x)]}
\end{equation}
and, since $\bar{r}(A)=n\P_n[\pderiv{z}L(y,F(x))\1_{A}(x)]/n(A)$ and $n(A)=n\P_n(x\in A)$, the score of a binary split $A=A_{0}\cup A_{1}$ defined in Equation~\eqref{eq:score} can be rewritten as
\begin{equation}\label{eq:score2}
 \Delta= \frac{\P_n[\pderiv{z}L(y,F(x))\1_{A_{0}}(x)]^2}{\P_n(x\in A_{0})}+\frac{\P_n[\pderiv{z}L(y,F(x))\1_{A_1}(x)]^2}{\P_n(x\in A_1)}.
\end{equation}

The stochastic algorithm associated with softmax binary splitting and  softmax gradient tree are summarized in Algorithms~\ref{algo1} and~\ref{algo2} respectively. We write $T(x;\P_n,F)$ to emphasize the dependency of the softmax gradient tree on  the sample distribution $\P_n$ and predictor $F$. 

\begin{algorithm}\label{algo1} Softmax binary splitting.
\begin{itemize}
\item Parameters:  $K\geq 1$ and $\beta\geq 0$.
\item Input: sample distribution $\P_n$, predictor $F$, region $A$.
\item Output: randomized partition  $A=A_{0}\cup A_{1}$.
\item Procedure:
\begin{enumerate}
\item Draw $(j_1,u_1),\dots, (j_K,u_K)$ independently uniformly on $\{1,\dots,p\}\times [0,1]$.
\item For $k=1,\ldots,K$, compute the partition $A^k=A_{0}^k\cup A_{1}^k$ according to~\eqref{eq:split} and the score $\Delta_k$ according to~\eqref{eq:score2}.
\item Randomly select $(j,u)$ according to the softmax distribution~\eqref{eq:softmax}.
\item Output the partition $A=A_{0}\cup A_{1}$ associated with $(j,u)$.
\end{enumerate}
\end{itemize}
\end{algorithm}

\begin{algorithm}\label{algo2} Softmax gradient tree $T(x;\P_n,F)$.
\begin{itemize}
\item Parameters:  $d\geq 1$, $K\geq 1$ and $\beta\geq 0$.
\item Input: sample distribution $\P_n$, predictor $F$.
\item Output: randomized tree function  $T(x;\P_n,F)$.
\item Procedure:
\begin{enumerate}
\item Construct a partition $(A_v)_{v\in\{0,1\}^d}$ using softmax binary splitting (Algorithm~\ref{algo1}) recursively with depth $d\geq 1$.
\item Compute the leaf values $\tilde r(A_v)$, $v\in\{0,1\}^d$, according to~\eqref{eq:leaf-values2}.
\item Output the softmax gradient tree $T(x;\P_n,F)$ defined by Equation~\eqref{eq:gradient-tree}.
\end{enumerate}
\end{itemize}
\end{algorithm}

\textbf{Infinitesimal gradient boosting}. Infinitesimal gradient boosting is defined as the vanishing-learning-rate limit of gradient boosting and characterized by a nonlinear ordinary differential equation in function space. The existence of a limit, for the algorithm of gradient boosting with learning rate $\lambda$, as $\lambda\to 0$, is justified in \cite{DD21} for a fixed input $(x_i,y_i)_{1\leq n}$. Let $\mathbb{B}$ denote the space of measurable bounded functions $F:[0,1]^p\to \mathbb{R}$ endowed with the supremum norm $\norm{\cdot}_\infty$.
The limiting dynamics that we obtain are driven by the infinitesimal boosting operator $\mathcal{T}_n:\mathbb{B}\to \mathbb{B}$, defined by
\begin{equation}\label{eq:ibo}
\mathcal{T}_n(F)(x)=\mathbb{E}[T(x;\P_n,F)],\quad F\in\mathbb{B}, x\in[0,1]^p,  
\end{equation}
where $T(x;\P_n,F)$ is the softmax gradient tree defined in Algorithm~\ref{algo2} and  expectation is taken with respect to the algorithm randomness (and not the sample randomness)
--- see \citet[Theorem~1.9]{DD21} for a precise formulation of the convergence of the gradient boosting algorithm to the infinitesimal gradient boosting as $\lambda\to 0$. Interestingly, the  operator $\mathcal{T}_n$ defined as the expectation of a randomized tree can be interpreted as an infinite random forest appearing when averaging a large number of randomly finite trees --- see e.g. \cite{SBV15}  for a discussion on the convergence of random forests when the number of trees goes to infinity.
Under mild assumptions (discussed below), the  operator $\mathcal{T}_n$ is locally Lipschitz in $F$ and infinitesimal gradient boosting $(\hat F_t^n)_{t\geq 0}$ is defined as the unique solution in $\mathbb{B}$ of the differential equation
\begin{equation}\label{eq:ode}
  \deriv{t}{F_t}=\mathcal{T}_n(F_t),\quad t\geq 0,
\end{equation}
with initialization at the constant function
\[
\hat F_0^n=\argmin_{z\in\mathbb{R}}\P_n[L(y,z)].
\]
Importantly, we have shown in \citet[Proposition~4.7]{DD21} that the training error $t\mapsto \P_n[L(y,\hat F_t^n(x))]$ is non-increasing  (which is a natural property resulting from the fact that gradient boosting is meant to minimize the training error) and that the mean residual $\P_n[\pderiv{z}L(y,\hat F_t^n(x))]\equiv 0$ is identically zero (which follows from the line search approximation in the definition of gradient trees).

In order to consider random samples and measurability issues, it is useful to work on a complete and separable function space rather than on the non-separable space~$\mathbb{B}$. The path $(\hat F_t^n)_{t\geq 0}$ a priori defined in $\mathbb{B}$ remains in a function space with strong regularity properties. For $q\in [1,+\infty]$, let $\mathbb{W}^q$ denote the space of functions $F:[0,1]^p\to \mathbb{R}$ of the form
  \[
    F:x\mapsto \int_{[0,x]}f_F(u)\,\pi_0(\rmd u),
  \]
  where $\pi_0$ is a reference probability distribution on $[0,1]^p$ (see Section~\ref{sec:preliminaries-algo} for more details), $f_F\in L^q(\pi_0)$ and $[0,x]=\{u\in [0,1]^p: 0\leq u\leq x\}$. We endow $\mathbb{W}^q$ with the natural Banach norm
  \[
    \norm{F-G}_{\mathbb{W}^q}= \norm{f_F-f_G}_{L^q(\pi_0)}.
  \]
Clearly, for all $1\leq q \leq q' \leq \infty$, we have $\mathbb{B}\supset \mathbb{W}^q\supset \mathbb{W}^{q'}$, and for any $F\in \mathbb{W}^{q'}$, we have $\norm{F}_\infty \leq \norm{F}_{\mathbb{W}^q}\leq \norm{F}_{\mathbb{W}^{q'}}$. The infinitesimal boosting operator has its image in $\mathbb{W}^\infty$ and this implies that infinitesimal gradient boosting $(\hat F_t^n)_{t\geq 0}$ can be seen as a smooth path in $\mathbb{W}^q$ for all $1\leq q\leq \infty$. 

For $1\leq q<\infty$, the Banach space $\mathbb{W}^q$ is separable, and measurability (and even continuity) properties will be established in Section~\ref{sec:preliminaries-measurability} that allow to consider infinitesimal gradient boosting with random sample. This paper studies the almost sure convergence of infinitesimal gradient boosting when the sample size goes to infinity.

\subsection{Assumptions}

We next specify our working assumptions, that are quite general and satisfied by the three cases considered in Example~\ref{examples}. We start with a convexity assumption on the loss function in its second variable.
\begin{assumption} \label{ass:L-regular-convex}
  The function $L:\mathcal{Y}\times\mathbb{R}\to \mathbb{R}$ is $C^2$, with $\pderiv[2]{z}L(y,z)$ positive and locally Lipschitz-continuous in $z$.
  Furthermore, for all $z\in\mathbb{R}$, we have
  \[
  \E\big[L(Y,z)\big]+\E\Big[\abs[\big]{\pderiv{z}L(Y,z)}\Big]+\E\Big[\pderiv[2]{z}L(Y,z)\Big] < \infty
  \]
  and the strictly convex map $z\mapsto \mathbb{E}[L(Y,z)]$ has a unique minimizer.
\end{assumption}

The next assumption requires some integrability of the residuals.

\begin{assumption} \label{ass:Y-integrable}
  There exists $q>1$ such that for any compact subset $K\subset \mathbb{R}$, we have
  \[
    \sup_{x\in[0,1]^p}\E\Big[\sup_{z\in K}\abs[\big]{\pderiv{z}L(Y,z)}^q \;\Big|\;X=x\Big] < \infty.
  \]
\end{assumption}
This assumption is trivially satisfied for classification, either with cross-entropy or exponential loss. In the case of regression it is equivalent to $\sup_x\E[\abs{Y}^q \mid X=x] < \infty$ for some $q>1$.

The two following assumptions are more technical and we believe they are not too stringent.
The first one only concerns the loss function.
\begin{assumption} \label{ass:bound-ratios}
One of the following conditions holds.
\begin{enumerate}[label=(\roman*)]
  \item For any compact subset $K\subset \mathbb{R}$, we have $\displaystyle
    \sup_{\substack{z\in K\\y\in \mathcal{Y}}}\abs[\Big]{\frac{\pderiv{z}L(y,z)}{\pderiv[2]{z}L(y,z)}} < \infty.
  $
  \item For any compact subset $K\subset \mathbb{R}$, we have $\inf_{(y,z)\in \mathcal{Y}\times K} \pderiv[2]{z}L(y,z) > 0$.
\end{enumerate}
\end{assumption}
Note that the first point above covers the classification case, while the second point covers the regression case.

Our last assumption involves the conditional distribution of $Y$ given $X$.
For $i\in\{1,2\}$, we define
\begin{equation} \label{eq:def-ell_i}
\ell_i(x,z) = \E\Big[\pderiv[i]{z}L(Y,z)\;\Big|\;X=x\Big], \qquad x\in [0,1]^p,z\in \mathbb{R}.
\end{equation}

\begin{assumption} \label{ass:bounded-lipshitz}
For any compact subset $K\subset \mathbb{R}$, there exists $C>0$ such that for all $i\in \{1,2\},\,x\in [0,1]^p$ and $z,z'\in K$,
\begin{gather*}
\abs{\ell_i(x,z)} \leq C, \quad \ell_2(x,z) \geq \frac{1}{C},\\
\text{and}\quad \abs{\ell_i(x,z)-\ell_i(x,z')}\leq C\abs{z-z'}.
\end{gather*}
\end{assumption}
In the regression case, \Cref{ass:bounded-lipshitz} is a consequence of \Cref{ass:Y-integrable}. In the classification case, it is trivially satisfied for the two cases we consider.

\begin{example}\label{examples-2}
The different assumptions involve the loss functions and its derivatives. We recall the corresponding formulas for the three main cases from Example~\ref{examples} for which the different assumptions are easily verified.
 \begin{itemize}
 \item regression with squared error:  $\mathcal{Y}=\mathbb{R}$ and 
 \[
 L(y,z)=\frac{1}{2}(y-z)^2,\quad \pderiv{z}{L}(y,z)=z-y,\quad \pderiv[2]{z}{L}(y,z)=1.
 \]
  \item binary classification with cross-entropy: $\mathcal{Y}=\{0,1\}$ and 
 \[
 L(y,z)=-yz+\log(1+e^z),\quad \pderiv{z}{L}(y,z)=-y+p,\quad \pderiv[2]{z}{L}(y,z)=p(1-p),
 \]
  with $p= e^z/(1+e^z)$.
   \item binary classification with exponential loss: $\mathcal{Y}=\{-1,1\}$ and
 \[
 L(y,z)=e^{-yz},\quad \pderiv{z}{L}(y,z)=-ye^{-yz},\quad \pderiv[2]{z}{L}(y,z)=e^{-yz}.
 \]
  \end{itemize}
\end{example}

\subsection{Convergence to the population limit} \label{sec:convergence-results}

Our main results give the large sample asymptotics for the infinitesimal  boosting operator $\mathcal{T}_n$ and infinitesimal gradient boosting $(\hat F_t^n)_{t\geq 0}$. Measurability issues with respect to the input sample $(X_i,Y_i)_{1\leq i\leq n}$ will be considered in Section~\ref{sec:preliminaries-measurability}. We first define the limiting object corresponding to population gradient boosting.

\begin{definition}\label{def:T} ~
\begin{itemize}
\item The population softmax gradient tree $T(x;\P,F)$ is defined as the output of Algorithm~\ref{algo2} where the sample distribution $\P_n$ is replaced by the population distribution $\P$. 
\item The population infinitesimal boosting operator $\mathcal{T}:\mathbb{B}\to\mathbb{W}^\infty$ is defined by
 \[
  \mathcal{T}(F)(x) = \mathbb{E}[T(x;\P,F)],\quad F\in\mathbb{B},  x\in[0,1]^p,
\]
with expectation taken with respect to the randomness of the stochastic algorithm.
\end{itemize}
\end{definition}
An equivalent, more formal definition of $\mathcal{T}$ is given in Section~\ref{sec:preliminaries-algo} where we also check that $\mathcal{T}(F)\in\mathbb{W}^\infty$ for all $F\in\mathbb{B}$. Let $\mathcal{C}_{bb}(\mathbb{B},\mathbb{W}^q)$ be the space of functions from $\mathbb{B}$ to $\mathbb{W}^q$ that are bounded on bounded sets and endowed with the topology of uniform convergence on bounded sets.

\begin{theorem}\label{thm:asymptotic_ibo}
Let $q>1$ satisfy \Cref{ass:Y-integrable}. We have the almost sure convergence, as $n\to\infty$,
\[
\mathcal{T}_n\stackrel{a.s.}\longrightarrow \mathcal{T} \quad \mbox{ in $\mathcal{C}_{bb}(\mathbb{B},\mathbb{W}^q)$}.
\]
\end{theorem}

Next we consider the dynamics associated with the population infinitesimal boosting operator which defines the population gradient boosting process.
\begin{proposition}\label{prop:ODE}
For all $F_0\in\mathbb{W}^\infty$,  the differential equation in the space $\mathbb{W}^\infty$
 \begin{equation}\label{eq:ODE}
  \deriv{t}{F_t}=\mathcal{T}(F_t),\quad t\geq 0,
  \end{equation}
  has a unique maximal solution  started at $F_0$ at time $0$. 
\end{proposition}

It is not clear under our general assumptions whether the solution of~\eqref{eq:ODE} is defined for all time $t\geq 0$ or explosion may occur in finite time. We denote by $t_{\max}$ the maximal time of definition of the solution $(F_t)_{t\geq 0}$. A linear growth condition ensures that the solution is defined for all time  $t\geq 0$.
\begin{lemma}\label{lem:tmax}
 If the loss function satisfies
 \begin{equation}\label{eq:infinite-tmax}
   \E\Bigg[\bigg|\frac{\pderiv{z}L(Y,z)}{\pderiv[2]{z}L(Y,z)}\bigg|\;\Big|\; X=x \Bigg]\leq A|z|+B,\quad  x\in[0,1]^p,z\in\mathbb{R},
 \end{equation}
for some constants $A,B\geq 0$, then $t_{\max}=+\infty$.
\end{lemma}

Note that Equation~\eqref{eq:infinite-tmax} holds in the case of regression (under \Cref{ass:Y-integrable}, which implies that $\sup_x\E[\abs{Y}\mid X=x]<\infty$) or classification with exponential loss. In the case of classification with cross-entropy, Assumption~\eqref{eq:infinite-tmax} is not fulfilled in general and we do not know whether $t_{\max}$ is finite or not.

We finally consider convergence of the gradient boosting process $(\hat F_t^n)_{t\geq 0}$. The space of continuous functions $\mathcal{C}([0,t_{\max}),\mathbb{W}^q)$ is endowed with the topology of uniform convergence on compact sets.
\begin{theorem} \label{thm:asymptotic_igb}
  Let $q>1$ satisfy \Cref{ass:Y-integrable}.
Consider $\hat F_0= \argmin_{z\in\mathbb{R}} \mathbb{E}[L(Y,z)]$. As $n\to\infty$,  we have the almost sure convergence 
\[
(\hat F_t^n)_{t\geq 0}\stackrel{a.s.}\longrightarrow (\hat F_t)_{t\geq 0} \quad \mbox{ in $\mathcal{C}([0,t_{\max}),\mathbb{W}^q)$,}
\]
where $(\hat F_t)_{ t\geq 0}$ denotes the unique solution of~\eqref{eq:ODE} started from $\hat F_0$.
\end{theorem}

\subsection{Properties of population infinitesimal gradient boosting}\label{sec:infinite-population-results}

Gradient boosting is designed in order to minimize the training error and it is indeed proved that the training error $t\mapsto P_n[L(y,\hat F_t^n(x))]$ is non-increasing, see Proposition~4.6 in \citet{DD21}. In the population limit, gradient boosting has the fundamental property that the test error is non increasing. 

\begin{proposition}\label{prop:decreasing_error}
For all initialization $\hat F_0\in\mathbb{B}$, the test error $t\mapsto \mathbb{E}[L(Y,\hat F_t(X))]$ is non-increasing.

\end{proposition}

The specific initialization $\hat F_0^n=\argmin_{z\in\mathbb{R}}\P_n[L(y,z)]$ ensures that the finite sample  model has centered residual on the training set, that is $t\mapsto \P_n[\pderiv{z}{L}(y,\hat F_t^n(x))]$ is identically null, see Proposition~4.6 in \citet{DD21}. Interestingly, this property is preserved in the population limit. 

\begin{proposition} \label{prop:zero_residual}
For $\hat F_0 = \argmin_{z\in\mathbb{R}} \mathbb{E}[L(Y,z)]$, the population gradient boosting has centered residuals on the population, that is $t\mapsto \mathbb{E}[\pderiv{z}{L}(Y,\hat F_t(X))]$ is identically null.
\end{proposition}

 We next focus on the long time behavior of population infinitesimal gradient boosting.  It is related to the  critical points of the ODE \eqref{eq:ODE} that we first characterize.   Let $\mathcal{A}_d$ denote  the class of hypercubes $A=[a,b]$, $a,b\in[0,1]^p$, $a\leq b$, where $[a_j,b_j]=[0,1]$ for all but at most $d$ dimensions. Clearly, the leaves of a gradient tree with depth $d$ belong to $\mathcal{A}_d$. For a subset of indices $J\subset\{1,\ldots,p\}$, we write $X_J=(X_j)_{j\in J}$. 

\begin{proposition}\label{prop:critical-points}
Let $F\in\mathbb{B}$. The following properties are equivalent:
\begin{enumerate}[label=(\roman*)]
\item $\mathcal{T}(F)=0$;
\item $\mathbb{E}\big[\pderiv{z}{L}(Y,F(X))\1_{A}(X)\big]=0$ for all $A\in \mathcal{A}_d$;
\item For all $J\subset \{1,\ldots,p\}$ with cardinal at most $d$, $\mathbb{E}\big[\pderiv{z}{L}(Y,F(X))\mid X_J\big]=0$ a.s.
\end{enumerate}
\end{proposition}

\begin{example}\label{ex:crit-point-reg} In the case of regression, $\pderiv{z}{L}(y,F(x))=F(x)-y$.  We consider $L^2=L^2([0,1]^p,\P_X)$ and the subspace $L^2_d=\overline{\mathrm{span}}(\1_{A},\ A\in\mathcal{A}_d)$ --- in other words it is the subset of functions $f\in L^2$ that can be written
  \[
  f(x_1,\dots,x_p) = \sum_{\substack{J\subset [\![1,p]\!]\\ \abs{J}\leq d}} f_J\big((x_i)_{i\in J}\big), \qquad x\in [0,1]^p.
  \]
Let $\mathcal{P}_d:L^2\to L^2$ be the orthogonal projection on $L^2_d$. We deduce from Proposition~\ref{prop:critical-points} that $\mathcal{T}(F)=0$ if and only if $\mathcal{P}_d(F)=\mathcal{P}_d(F^*)$ $\P_X$-a.s., where $F^*$ denotes the regression function $F^*(x)=\mathbb{E}[Y\mid X=x]$. In other words, the set of critical points is exactly the affine subspace orthogonal to $L^2_d$ containing $\mathcal{P}_d(F^*)$. When $d\geq p$, then $L^2_d=L^2$ and  $\mathcal{P}_d(F^*)=F^*$ is the only critical point. Furthermore, since  population infinitesimal gradient boosting remains in the subspace $L^2_d$, only critical points in $L^2_d$ are relevant. The only critical point in $L^2_d$ is the projection $\mathcal{P}_d(F^*)$ which is also the minimizer of the squared error $\mathbb{E}[(Y-F(X))^2]$ under the constraint $F\in L^2_d$. 
\end{example}

For a general loss function and assuming $d\geq p$, using point (iii) of \Cref{prop:critical-points} and the convexity of $L$ in its second variable, we see that we have $\mathcal{T}(F)=0$ if and only if
\[
F(x)=\argmin_{z\in\mathbb{R}} \mathbb{E}[L(Y,z)|X=x]\quad \mbox{$\P_X$-a.s.}
\]
In statistical learning, a desirable general property is consistency, which means that the test error converges to its minimum. Such consistency for population boosting is considered in \cite{B04}  for a version of Adaboost. Proposition~\ref{prop:critical-points} shows that the critical points of the ODE are exactly the minimizer of  the expected loss, which is a first step toward consistency. Unfortunately, it seems difficult to prove formally consistency and we are able only to prove a weaker statement.

 We focus on the case of regression  where consistency is equivalent to the convergence $\hat F_t\to \mathcal{P}_d(F^*)$ in $L^2=L^2([0,1]^p,\P_X)$ with the notation introduced in \Cref{ex:crit-point-reg}. We cannot prove strong convergence but weak convergence only. We recall that in the Hilbert space $L^2$, a sequence  $G_n$ converges weakly to $G$, noted $G_n\overset{w}{\to}G$, if the convergence of inner products $\langle G_n,H\rangle\to \langle G,H\rangle$ holds for all $H\in L^2$. 
\begin{proposition} \label{prop:weak-consistency} ~
\begin{enumerate}[label=(\roman*)]
  \item In the case of regression,  weak convergence holds:
  \[
  \hat{F}_t \overset{w}{\longrightarrow} \mathcal{P}_d(F^*),\quad \mbox{as $t\to\infty$}.
  \]
  \item For completely random trees (case $\beta=0$), strong convergence holds:
  \[
  \hat{F}_t \overset{L^2}{\longrightarrow} \mathcal{P}_d(F^*),\quad \mbox{as $t\to\infty$}.
  \]
\end{enumerate}
\end{proposition}

\begin{remark}
\Cref{prop:weak-consistency} states that the population infinitesimal gradient boosting behaves quite well for estimating the regression function: when $d\geq p$, it satisfies $\hat F_t\to F^*$  as $t\to\infty$ in some weak sense. However,  we only have access in practice to finite samples and we must use $(\hat F_t^n)$ to approximate $F^*$. It is known that finite-population (infinitesimal) gradient boosting is prone to overfitting \citep[see][Proposition~4.11]{DD21} so that the limit $t\to\infty$ does not commute with the large sample limit  $n\to\infty$ from \Cref{thm:asymptotic_igb}.

In practice overfitting can be avoided thanks to early stopping  and $\hat F_{t_n}^n$ is used as an approximation of $F^*$. We can deduce from \Cref{thm:asymptotic_igb}  the existence of a (sufficiently slowly) increasing sequence $t_n\to \infty$ such that $\sup_{t\in [0,t_n]} \abs{\hat{F}_t^n - \hat{F}_t} \to 0$ almost surely. Identifying the growth rate of a sequence $(t_n)$ ensuring this property would be a significant advance in the study of sample-based gradient boosting but we have not managed to do so with the techniques we develop.
\end{remark}

\begin{remark}
We would expect that strong convergence holds in the case $\beta>0$ as well; a hint in this direction is the following remark.
Let us temporarily write $\mathcal{T}_\beta$ instead of $\mathcal{T}$ to highlight the dependence on the parameter $\beta$, and consider $(\tilde{F}_t)_{t\geq 0}$ the solution of $\deriv{t}\tilde{F}_t=\mathcal{T}_\beta(\tilde{F}_t)$, started from some $\tilde{F}_0\in L^2_d$ rather than from the constant $F_0$.
In the case of regression, for the proofs of the results above we will show and use the fact that
\[
  \deriv{t}\norm[\big]{\tilde{F}_t-\mathcal{P}_d(F^*)}_{L^2}^{2} \,\big|_{t=0} = 2\langle \mathcal{P}_d(F^*) -\tilde{F}_0, \mathcal{T}_\beta(\tilde{F}_0)\rangle \; \leq 0.
\]
Then one can show, after tedious calculations, that the latter quantity is decreasing when $\beta$ increases.
This suggests that, at least around $t=0$, $\norm[\big]{\tilde{F}_t-\mathcal{P}_d(F^*)}_{L^2}^{2}$ tends to $0$ faster when $\beta>0$ than when $\beta=0$.
However, it is not obvious to compare the whole trajectories of $(\tilde{F}_t)_{t\geq 0}$ as a function of $\beta$ and with the techniques we used, we were not able to prove the convergence $\hat{F}_t \overset{L^2}{\longrightarrow} \mathcal{P}_d(F^*)$ in the case $\beta>0$.
\end{remark}

Similar results may be expected in the general case, for instance in classification, where we expect that the test error converges to the minimal risk over the space in which $(\hat{F}_t)_{t\geq 0}$ lives.
This remark leads us to the following conjecture.

\begin{conjecture}~
\begin{enumerate}[label=(\roman*)]
\item In the case of regression, for $\beta> 0$, strong convergence holds
\[
\hat{F}_t \overset{L^2}{\longrightarrow} \mathcal{P}_d(F^*),\quad \mbox{as $t\to\infty$}.
\]
\item In the general case, for $\beta\geq 0$,
\[
  \E[L(Y,\hat{F}_t(X))] \longrightarrow \inf_{F\in \mathbb{W}^\infty\cap L^2_d}\E[L(Y,F(X))], \quad  \mbox{as $t\to\infty$}.
\]
\end{enumerate}
\end{conjecture}

\section{Preliminaries}\label{sec:preliminaries}
We now present some technical results that will be used for the proofs of the main results from Section~\ref{sec:main-results}.
The proofs of the most technical results in \Cref{sec:preliminaries} (Sections~\ref{sec:preliminaries-w-norm} and~\ref{sec:prelim-glivenko-cantelli}) are provided in Appendix~\ref{sec:proofs-preliminaries}, while the other proofs are deferred to \Cref{sec:proofs-main-results}.

\subsection{Explicit formulas associated with Algorithm~\ref{algo2}}\label{sec:preliminaries-algo}
We recall some technical background and explicit formulas  associated with Algorithm~\ref{algo2}. We refer to \citet[Section~2.2]{DD21} for more details. 

The binary rooted tree with depth $d\geq 1$ (from graph theory) is defined on the vertex set $\mathscr{T}_d=\cup_{l=0}^d \{0,1\}^l$. The vertex set is divided into the internal nodes $v\in \mathscr{T}_{d-1}$ and the terminal nodes $v\in\{0,1\}^d$, also called leaves. The construction of the partition in Algorithm~\ref{algo2} starts from the  the single component $A_\emptyset=[0,1]^p$ indexed by the root $\emptyset$ and performs softmax binary splitting recursively with depth $d$ so as to end up with a partition $(A_v)_{v\in\{0,1\}^d}$ indexed by the leaves. This can be encoded thanks to the notion of \textit{splitting scheme} 
\[
\xi=(j_v,u_v)_{v\in\mathscr{T}_{d-1}}\in ([\![1,p]\!]\times (0,1))^{\mathscr{T}_{d-1}}
\]
giving the covariate and threshold used at each internal node to perform the split.  When $\beta=0$, the softmax distribution~\eqref{eq:softmax} is the uniform distribution so that the splits $(j_v,u_v)$ are independent and uniform on $[\![1,p]\!]\times (0,1)$. This situation corresponds to a completely random tree and we denote by $Q_0$ the distribution of the associated splitting scheme. When $\beta>0$, the distribution of the splitting scheme $\xi$ depends on the sample distribution $\P_n$ and model $F$ and is denoted by $Q_{n,F}$. According to Proposition~2.1 in \cite{DD21}, softmax binary splitting as defined in Algorithm~\ref{algo1}, implies 
\begin{equation}\label{eq:RN}
\frac{\rmd Q_{n,F}}{\rmd Q_0}(\xi^1)= 
\int \prod_{v\in \mathscr{T}_{d-1}}  \frac{K\exp(\beta \Delta(j_v^1,u_v^1; A_v))}{\sum_{k=1}^K \exp(\beta \Delta(j_v^{k},u_v^{k}; A_v))} \,Q_0(\rmd \xi^2)\cdots Q_0(\rmd \xi^K),
\end{equation}
with $(A_v)_{v\in\mathscr{T}_d}$ the different regions induced by the splitting scheme $\xi^1$, $\Delta(j,u;A)$ the score resulting from the split at $(j,u)$ of region $A$, and $\xi^k=(j_v^k,u_v^k)_{v\in\mathscr{T}_{d-1}}$, $1\leq k\leq K$. In this formula, the different scores $\Delta$ implicitly depend on $F$ and on the sample distribution $\P_n$ according to Equation~\eqref{eq:score2}. Note that in Equation~\eqref{eq:RN}, each factor in the product is bounded by $K$ so that the Radon-Nikodym derivative~\eqref{eq:RN} satisfies
\begin{equation}\label{eq:RN-bounded}
\Big|\frac{\rmd Q_{n,F}}{\rmd Q_0}(\xi)\Big|\leq K^{2^d-1}.
\end{equation}

Once the distribution of the splitting scheme is made explicit, one can easily deduce an integral formula for the infinitesimal boosting operator defined by \eqref{eq:ibo}. Indeed, Equations~\eqref{eq:gradient-tree} and~\eqref{eq:leaf-values2} together imply
\[
T(\cdot;\P_n,F)=-\sum_{v\in\{0,1\}^d} \frac{\P_n[\pderiv{z}{L}(y,F(x))\1_{A_v}(x)]}{\P_n[ \pderiv[2]{z}{L}(y,F(x))\1_{A_v}(x)]}\1_{A_v}(\cdot),
\]
where $(A_v)_{v\in\{0,1\}^d}$ denotes the random partition driven by a splitting scheme with distribution $Q_{n,F}$. Integrating with respect to the splitting scheme, we deduce
\begin{equation}\label{eq:ibo2}
\mathcal{T}_n(F)=-\int \sum_{v\in\{0,1\}^d} \frac{\P_n[\pderiv{z}{L}(y,F(x))\1_{A_v}(x)]}{\P_n[ \pderiv[2]{z}{L}(y,F(x))\1_{A_v}(x)]}\1_{A_v}(\cdot)\frac{\rmd Q_{n,F}}{\rmd Q_0}(\xi) \, Q_0(\rmd \xi).
\end{equation}

\begin{remark}
In \Cref{def:T}, \Cref{algo2} is also considered when the sample distribution $\P_n$ is replaced by the population distribution $\P$. Equations~\eqref{eq:RN} and~\eqref{eq:ibo2} still hold true with straightforward modifications ($\P_n$ replaced by $\P$).
\end{remark}

\subsection{Norm estimate and regularity in \texorpdfstring{$\mathbb{W}^q$}{Wq}} \label{sec:preliminaries-w-norm}

The space $\mathbb{W}^q$ was introduced in the end of \Cref{sec:setting} with no details on the reference measure $\pi_0$. We now provide further details as well as useful properties.
Recall that, for $q\in [1,+\infty]$,  $\mathbb{W}^q$ denotes the space of functions $F:[0,1]^p\to \mathbb{R}$ of the form
\[
F:x\mapsto \int_{[0,x]}f_F(u)\,\pi_0(\rmd u),
\]
where $\pi_0$ is a reference probability distribution on $[0,1]^p$ and $f_F\in L^q(\pi_0)$. Naturally, $\mathbb{W}^q$ is endowed with the norm $\norm{F}_{\mathbb{W}^q}=\norm{f_F}_{L^q(\pi_0)}$. The reference probability distribution $\pi_0$ is related to completely random trees and to their splitting scheme distribution denoted by $Q_0$.
The idea is that $\pi_0$ should distribute its mass on the typical points where the tree functions $T(\cdot;\P_n,F)$ vary, and we take these points as the ``corners'' of the tree leaves (\Cref{lem:bound-norm-Wq} below justifies this choice).
To define formally $\pi_0$, recall that
a splitting scheme $\xi$ induces a partition of $[0,1]^d$ into leaves, denoted $(A_v)_{v\in\{0,1\}^d}$. For a hyperrectangle $A=[a,b]\subset [0,1]^p$, we define  the set $c(A)=\prod_{j=1}^p \{a_j,b_j\}$ of its $2^d$ vertices --- or corners. The point measure
\[
	\pi_\xi = \sum_{v\in A_v}\sum_{x\in c(A_v)}\1_{[0,1)^d}(x) \delta_x
\]
is seen as a random point process on $[0,1)^d $ under the distribution $Q_0(\rmd \xi)$.
We then denote by $\tilde\pi_0$ its intensity measure and by $\pi_0$ the  normalized probability distribution, that is: for all Borel subset $B\subset [0,1]^p$,
  \begin{equation} \label{eq:def-pi0}
    \tilde{\pi}_0(B) = \int \pi_\xi(B) \,Q_0(\rmd \xi) \quad \text{and}\quad \pi_0(B) = \frac{\tilde{\pi}_0(B)}{\tilde{\pi}_0([0,1]^p)}.
  \end{equation}
Note that in \citet[Section 4.4]{DD21}, the space $\mathbb{W}^\infty$ was used under the notation $\mathbb{W}$, and the normalization of $\pi_0$ (which is useful to simplify some formulas) was not considered.

In view of Equation~\eqref{eq:ibo2}, the following technical lemma will be crucial to obtain norm estimates  in $\mathbb{W}^q$ for the gradient boosting operator.
\begin{lemma} \label{lem:bound-norm-Wq}
Let $T:[0,1]^p\to\mathbb{R}$ be a function of the form
\[
  T(z) = \int \sum_{v\in\{0,1\}^d}\psi(v,\xi)\1_{A_v}(z)\,Q_0(\rmd \xi).
\]
Then, for all $q\in [1,\infty]$, we have
\[
  \norm{T}_{\mathbb{W}^q} \;\leq\; 2^{p+d}\sum_{v\in\{0,1\}^d} \norm{\psi(v,\cdot)}_{L^q(Q_0)}.
\]
\end{lemma}

Next we describe a regularity property of functions in $\mathbb{W}^q$. Recall the definition of the measure $\pi_0$ in~\eqref{eq:def-pi0}.
It is interesting to note, as a result of \cite{DD21}, that $\pi_0$ is absolutely continuous with respect to the measure
\[
\sum_{\substack{J\subset [\![1,p]\!]\\
    \abs{J}\leq d}} \sum_{\epsilon\in \{0,1\}^{J^c}} \leb_{J,\epsilon},
\]
where $\leb_{J,\epsilon}$ is the $\abs{J}$-dimensional Lebesgue measure on the subspace
\[
\{x\in [0,1]^p : \forall j\in J, x_j\in [0,1],\; \forall j\notin J, x_j=\epsilon_j\}.
\]
This is the key to estimate the modulus of continuity of functions in $\mathbb{W}^q$.

\begin{proposition} \label{prop:W-to-uniform-norm}
  Let $F\in \mathbb{W}^{q}$ for $q\in [1,\infty]$.  Then $F$ is continuous on $[0,1]^p$ and its modulus of continuity satisfies
  \[
  \omega_F(\delta) := \sup_{\substack{x,y\in[0,1]^p\\\norm{x-y}_\infty \leq \delta}} \abs{F(x)-F(y)} \;\leq \; \Big( C \delta\big(1-\log \delta \big)^{d-1}\Big)^{\frac{1}{q'}} \norm{F}_{\mathbb{W}^{q}},
  \]
  where $q'=q/(q-1)$ and $C$ is a constant that depends only on $p$ and $d$.
\end{proposition}

\subsection{Properties of infinitesimal boosting operators}\label{sec:preliminaries-measurability}

Recall that infinitesimal boosting operator $\mathcal{T}_n$ is defined in Equation~\eqref{eq:ibo} and its population version $\mathcal{T}$ in Definition~\ref{def:T}. Explicit formulas involving splitting scheme are provided in \Cref{sec:preliminaries-algo}.  Together with \Cref{lem:bound-norm-Wq}, Equations~\eqref{eq:RN}-\eqref{eq:ibo2} are crucial in our analysis of the infinitesimal boosting operators.

\begin{proposition}\label{prop:ibo_Lipschitz}
For every $F\in\mathbb{B}$, $\mathcal{T}_n(F)\in\mathbb{W}^\infty$ and  $\mathcal{T}(F)\in\mathbb{W}^\infty$. Furthermore, when restricted on an arbitrary bounded set,  the mappings $\mathcal{T}_n:\mathbb{B}\to\mathbb{W}^\infty$ and $\mathcal{T}:\mathbb{B}\to\mathbb{W}^\infty$ are Lipschitz-continuous.
\end{proposition}

Note that the result for $\mathcal{T}_n$ was already stated in \citet[Lemma~5.4]{DD21} and we extend it here naturally to $\mathcal{T}$. In this former work, the gradient boosting operator $\mathcal{T}_n$ and the gradient boosting process $(\hat F_t^n)_{t\geq 0}$ were considered for a fixed input sample  $(\mathbf{x},\mathbf{y})=(x_i,y_i)_{1\leq i\leq n}$.
In this paper, we consider a random independent sample of size $n$ and we denote by $\mathcal{T}_n(\cdot;\mathbf{x},\mathbf{y})$  and $\hat F_t^n(\cdot;\mathbf{x},\mathbf{y})$ the corresponding random operator and random process.
Therefore we need to check the measurability (and even prove the continuity) of $\mathcal{T}_n(\cdot;\mathbf{x},\mathbf{y})$ and $\hat F_t^n(\cdot;\mathbf{x},\mathbf{y})$  as functions of the input sample $(\mathbf{x},\mathbf{y})$.
Recall that we denote by $\mathcal{C}_{bb}(\mathbb{B}, \mathbb{W}^q)$ the space of continuous functions $\mathbb{B}\to \mathbb{W}^q$ that are bounded on bounded sets, endowed with the topology of uniform convergence on bounded sets.
We also endow $\mathcal{C}([0,\infty),\mathbb{W}^q)$ with the topology of uniform convergence on compact intervals.
\begin{proposition}\label{prop:continuity_ibo_igb} Let $q\in [1,+\infty)$
\begin{enumerate}[label=(\roman*)]
\item The mapping
\[
(\mathbf{x},\mathbf{y})\in ([0,1]^p\times \mathbb{R})^n \;\longmapsto\; \mathcal{T}_n(\cdot;\mathbf{x},\mathbf{y})\in \mathcal{C}_{bb}(\mathbb{B},\mathbb{W}^q)
\]
is continuous.
\item The mapping
\[
(\mathbf{x},\mathbf{y})\in ([0,1]^p\times \mathbb{R})^n \;\longmapsto\; \big(\hat F_t^n(\cdot;\mathbf{x},\mathbf{y})\big)_{t\geq 0}\in \mathcal{C}([0,\infty),\mathbb{W}^q)
\]
is continuous.
\end{enumerate}
\end{proposition}

\subsection{Glivenko--Cantelli classes} \label{sec:prelim-glivenko-cantelli}

Our main results, Theorems~\ref{thm:asymptotic_ibo} and~\ref{thm:asymptotic_igb}, state the almost sure convergence of the infinitesimal boosting operator and infinitesimal gradient boosting process as the sample size tends to infinity. The main technical tool for our proof is the notion of Glivenko--Cantelli classes of function. Following \citet[Section~19.2]{vW96}, a class $\mathcal{F}$ of measurable functions $f:(x,y)\mapsto f(x,y)$ is said $\P$-Glivenko--Cantelli when
\[
\sup_{f\in\mathcal{F}} |\P_n[f(x,y)]-\P[f(x,y)]|\stackrel{as*}{\longrightarrow} 0, \quad \mbox{as $n\to\infty$},
\]
where $\P_n=\sum_{i=1}^n\delta_{(X_i,Y_i)}$ is the empirical measure associated with an i.i.d. sample $(X_i,Y_i)_{i\geq 1}$ with distribution $\P$. The notation $\stackrel{as*}{\longrightarrow} 0$ stands for  almost sure convergence in outer probability, which is introduced to handle the possible non-measurability of the supremum.

In the context of gradient boosting, the following result will be useful. We denote by $\mathcal{A}$ the class of hyperrectangles of the form $A=[a,b]$ for some $a,b\in [0,1]^p$, $a\leq b$.

\begin{proposition}\label{prop:GC}
Let $q\in (1,\infty]$ and $B\subset\mathbb{W}^q$ a bounded set. The classes of functions
\[
\mathcal{F}_1=\left\{(x,y)\mapsto \pderiv{z}L(y,F(x))\1_{A}(x)\ :\ A\in\mathcal{A}, F\in B
\right\} \]
\[
\mathcal{F}_2=\left\{(x,y)\mapsto \pderiv[2]{z}L(y,F(x))\1_{A}(x)\ :\ A\in\mathcal{A}, F\in B\right\}
\]
are $\P$-Glivenko--Cantelli.
\end{proposition}

\section{Proofs of the main results}\label{sec:proofs-main-results}
\subsection{Proofs of Propositions~\ref{prop:ibo_Lipschitz},~\ref{prop:continuity_ibo_igb} and of Theorem~\ref{thm:asymptotic_ibo}}

The three results we consider here deal with continuity properties of the infinitesimal boosting operators $\mathcal{T}_n$ and $\mathcal{T}$, and can be proven similarly.
In order to ease the proof of Proposition~\ref{prop:ibo_Lipschitz},~\ref{prop:continuity_ibo_igb} and \Cref{thm:asymptotic_ibo} and treat in an unified way the finite sample case (associated with the measure $\P_n$) and the population case (associated with the measure $\P$), we introduce some notation.
Let $\mathcal{M}$ be the set of probability measures $\mu$ on the space $[0,1]^p\times \mathbb{R}$ satisfying Assumptions \ref{ass:L-regular-convex} to \ref{ass:bounded-lipshitz} (when seen as joint distributions for a pair $(X,Y)$ of random variables). The key for factorizing the proof  is that $\mathcal{M}$ contains any empirical distribution so that $\P_n,\P\in \mathcal{M}$. We use the short notation
\[
\mu(x\in A)=\int \1_{A}(x) \,\mu(\rmd x\rmd y) \quad \text{and}\quad \mu[f(x,y)]=\int f(x,y)\,\mu(\rmd x\rmd y).
\]

Quite generally, we may consider Algorithms~\ref{algo1} and~\ref{algo1} when $\P_n$ is replaced by a generic measure $\mu\in\mathcal{M}$. The corresponding gradient tree is written $T(x;\mu,F)$ and the infinitesimal boosting operator is written $\mathcal{T}_\mu$ (see \Cref{def:T}). In particular, $\mathcal{T}=\mathcal{T}_\P$ and $\mathcal{T}_n=\mathcal{T}_{\P_n}$. All the quantities and notation introduced so far can be adapted in a straightforward way: the leaf values \eqref{eq:leaf-values2} become
\[
\tilde{r}_\mu(F,A) = - \frac{\mu[\pderiv{z}L(y,F(x))\1_{A}(x)]} {\mu[\pderiv[2]{z}L(y,F(x))\1_{A}(x)]},
\]
where we have added the dependency on $F$ (so that $\tilde r(A)=\tilde{r}_{\P_n}(F,A)$); the score \eqref{eq:score2} attributed to the split of a region $A$ along variable $j$ at threshold $u$ becomes
\[
 \Delta_\mu(j,u;A)= \frac{\mu[\pderiv{z}L(y,F(x))\1_{A_{0}}(x)]^2}{\mu(x\in A_{0})}+\frac{\mu[\pderiv{z}L(y,F(x))\1_{A_1}(x)]^2}{\mu(x\in A_1)};
\]
the distribution of the splitting scheme is written $Q_{F}^\mu$ and Equations~\eqref{eq:RN}-\eqref{eq:ibo2} are readily modified.
In particular, Equation~\eqref{eq:ibo2} becomes
\begin{equation}\label{eq:ibo3}
\mathcal{T}_\mu(F)(z)=\int \sum_{v\in\{0,1\}^d} \tilde{r}_\mu(F,A) \frac{\rmd Q_{F}^\mu}{\rmd Q_0}(\xi) \1_{A_v}(z) \, Q_0(\rmd \xi).
\end{equation}
We can see that this is exactly the form required in Lemma~\ref{lem:bound-norm-Wq} to get $\mathbb{W}^q$-norm estimates.

\subsubsection{Proof of Proposition~\ref{prop:ibo_Lipschitz}}

The statement concerns $\mathcal{T}_\mu$ with either $\mu=\P$ or $\mu=\P_n$ and our proof holds for a generic $\mu\in\mathcal{M}$.  It is convenient to first state two technical lemmas. 

\begin{lemma} \label{lem:1-prop-ibo-lipschitz}
  When restricted on a bounded set, the map $F\in\mathbb{B} \mapsto \tilde r_\mu(F,A) $
  is bounded and Lipschitz, with constants that do not depend on $A$.
\end{lemma}
\begin{proof}
  Let us fix $M$. We note $\mathbb{B}_M:=\{F\in\mathbb{B},\norm{F}_\infty \leq M\}$ the ball with radius $M$ and consider $F,G\in \mathbb{B}_M$.
  With the notation of \Cref{ass:bounded-lipshitz}, we have
  \[
  \mu\Big[\pderiv[i]{z}{L}(y,F(x))\1_{A}(x)\Big] = \int_A \ell_i(x, F(x)) \,\mu_X(\rmd x),\qquad i\in \{1,2\}, 
  \]
  with $\mu_X$ the marginal distribution of $X$.   Using this and putting everything on the same denominator, we can write
  \begin{align*}
    &\abs{\tilde r_\mu(F,A) - \tilde r_\mu(G,A)}\\
    &=\Big(\int_{A^2} \ell_1(x, F(x))\ell_2(x', G(x')) - \ell_1(x, G(x))\ell_2(x', F(x')) \,\mu_X^{\otimes 2}(\rmd x\rmd x')\Big)\\
    &\quad \times \Big(\int_{A^2} \ell_2(x, F(x))\ell_2(x', G(x')) \,\mu_X^{\otimes 2}(\rmd x\rmd x')\Big)^{-1}.
  \end{align*}
  By \Cref{ass:bounded-lipshitz}, there is a constant $C>0$ that depends only on $M$ such that for all $i\in \{1,2\}$,
  \begin{gather*}
    \abs{\ell_i(x,F(x))}\vee \abs{\ell_i(x,G(x))} \leq C,\quad {\ell_2(x,F(x))}\wedge {\ell_2(x,G(x))} \geq C^{-1},\\
    \text{and }\quad \abs{\ell_i(x,F(x))-\ell_i(x,G(x))} \leq C \abs{F(x)-G(x)}.
  \end{gather*}
  We thus easily obtain
  \[
  \abs{\tilde r_\mu(F,A))} \leq C^2 \text{ and }\abs{\tilde r_\mu(F,A)-\tilde r_\mu(G,A)} \leq C^4 \norm{F-G}_\infty,
  \]
  which concludes the proof.
\end{proof}

\begin{lemma} \label{lem:2-prop-ibo-lipschitz}
 When restricted on a bounded set, the map $F\in \mathbb{B}\mapsto\frac{\rmd Q_{F}^\mu}{\rmd Q_0}(\xi)$ is bounded by $K^{2^d-1}$ and Lipschitz, with a constant that does not depend on $\xi$.
\end{lemma}
\begin{proof}
  Recall the expression for $\frac{\rmd Q_{F}^\mu}{\rmd Q_0}$ that can be deduced from \eqref{eq:RN}:
  \begin{equation}\label{eq:RN2}
  \frac{\rmd Q_{F}^\mu}{\rmd Q_0}(\xi^1)= 
  \int \prod_{v\in \mathscr{T}_{d-1}}  \frac{K\exp(\beta \Delta_\mu(j_v^1,u_v^1; A_v))}{\sum_{k=1}^K \exp(\beta \Delta_\mu(j_v^{k},u_v^{k}; A_v))} \,Q_0(\rmd \xi^2)\cdots Q_0(\rmd \xi^K),
  \end{equation}
  with 
  \begin{equation} \label{eq:Delta_i-expression}
    \Delta_\mu(j_v^{k},u_v^{k}; A_v) = \frac{\mu[\pderiv{z}L(y,F(x))\1_{A_{v0}^k}(x)]^2}{\mu(x\in A_{v0}^k)}+\frac{\mu[\pderiv{z}L(y,F(x))\1_{A_{v1}^k}(x)]^2}{\mu(x\in A_{v1}^k)},
  \end{equation}
  where $A_{v0}^k$ and $A_{v1}^k$ are the regions resulting from the $(j_v^k,u_v^k)$-split of $A_v$.
  
  Now note that the softmax function $(t_1,\dots,t_K) \mapsto \exp(t_1)/(\sum_{i=1}^K\exp(t_i))$  is bounded by $1$. Hence the product in~\eqref{eq:RN2}  has 
  $2^{d}-1$ factors bounded by $K$ so that the bound $K^{2^d-1}$ is clear.
 On the other hand, the softmax function is $1/2$-Lipschitz-continuous for the uniform norm. It is therefore sufficient to show that the maps $F\mapsto \Delta_\mu(j_v^{k},u_v^{k}; A_v)$ are bounded and Lipschitz-continuous on bounded subsets of $\mathbb{B}$ with constants that do not depend on $(j_v^k,u_v^k)$ and $A_v$. Using
\[
\frac{\mu[\pderiv{z}L(y,F(x))\1_{A}(x)]^2}{\mu(x\in A)}=\Big(\int_{A^2} \ell_1(x,F(x))\ell_1(x',F(x')) \,\mu_X^{\otimes 2}(\rmd x\rmd x')\Big)\times \Big( \int_A \mu_X(\rmd x)\Big)^{-1}, 
\] 
 this follows from \Cref{ass:bounded-lipshitz} with a similar argument as in the previous proof.
\end{proof}

We are now ready to prove \Cref{prop:ibo_Lipschitz}.

\begin{proof}[Proof of \Cref{prop:ibo_Lipschitz}]
  By Equation~\eqref{eq:ibo3}, 
  \[
   \mathcal{T}_\mu(F)(z)= \int \sum_{v\in \{0,1\}^d} \psi_{F}^\mu(v,\xi)\1_{A_v}(z) \,Q_0(\rmd \xi)
  \]
  with 
  \[
  \psi_{F}^\mu(v,\xi)=\tilde r_\mu(F,A_v)\frac{\rmd Q_{F}^\mu}{\rmd Q_0}(\xi).
  \]
  Let $M>0$ and assume $F,G\in \mathbb{B}_M$. \Cref{lem:bound-norm-Wq} implies
  \begin{equation}\label{eq:technical}
  \norm{\mathcal{T}(F)-\mathcal{T}(G)}_{\mathbb{W}^{\infty}} \leq 2^{p+d}\sum_{v\in\{0,1\}^d}\norm[\big]{\psi_{F}^\mu(v,\xi)-\psi_{G}^\mu(v,\xi)}_{L^\infty(Q_0)}.
  \end{equation}
   By \Cref{lem:1-prop-ibo-lipschitz} and \Cref{lem:2-prop-ibo-lipschitz}, we know that for  fixed $v,\xi$, the maps $F\mapsto \tilde r_\mu(F,A_v)$ and $F\mapsto \frac{\rmd Q_{F}^\mu}{\rmd Q_0}(\xi)$ are both bounded and Lipschitz on $\mathbb{B}_M$, for constants that do not depend on $v,\xi$.
  This implies that, for each $v$, the product $F\mapsto \psi_{F}^\mu(v,\xi)\in L^\infty(Q_0)$ is also bounded and Lipschitz on $\mathbb{B}_M$, thus concluding the proof.
\end{proof}

\subsubsection{Proof of Proposition~\ref{prop:continuity_ibo_igb}}
We first state a deterministic convergence lemma for sequences of infinitesimal gradient boosting operators. This lemma  will be key in the proof of \Cref{prop:continuity_ibo_igb} and also in the proof of \Cref{thm:asymptotic_ibo}.

\begin{lemma}\label{lem:deterministic-T-cv}
  Let $\mu, (\mu_k)_{k\geq 1}$ be distributions in $\mathcal{M}$. 
  Assume that there exists a bounded subset $B\subset \mathbb{B}$ and $q \geq 1$ such that:
  \begin{enumerate}[label=(\roman*)]
    \item $\displaystyle \sup_{F\in B}\sum_{v\in\{0,1\}^d} \abs[\big]{\tilde r_{\mu_k}(F,A_v)-\tilde r_{\mu}(F,A_v)} \to 0$, $Q_0$-a.s.
    \item $\displaystyle \sup_{F\in B}\sum_{v\in\{0,1\}^d} \abs[\big]{\frac{\mu_k[\pderiv{z}L(y,F(x))\1_{A_v}(x)]^2}{\mu_k(x\in A_v)}-\frac{\mu[\pderiv{z}L(y,F(x))\1_{A_v}(x)]^2}{\mu(x\in A_v)}} \to 0$, $Q_0$-a.s.
    \item $\displaystyle\int \sup_{k\in\mathbb{N},F\in B}\sum_{v\in\{0,1\}^d}\abs[\big]{\tilde r_{\mu_k}(F,A_v)}^q \,Q_0(\rmd \xi) < \infty$.
  \end{enumerate}
  Then $\sup_{F\in B}\norm{\mathcal{T}_{\mu_k}(F)-\mathcal{T}_\mu(F)}_{\mathbb{W}^q} \to 0$.
\end{lemma}

\begin{proof}
  Similarly as for Equation~\eqref{eq:technical}, \Cref{lem:bound-norm-Wq} entails
  \[
  \norm[\big]{\mathcal{T}_{\mu_k}(F)-\mathcal{T}_\mu(F)}_{\mathbb{W}^q} \leq 2^{p+d} \sum_{v\in\{0,1\}^d} \norm[\big]{\psi_F^{\mu_k}(v,\xi)-\psi_F^\mu(v,\xi)}_{L^q(Q_0)}.
  \]
  Using the triangle inequality, we have the bound
  \[
  \sup_{F\in B}\abs[\big]{\psi_F^{\mu_k}(v,\xi)-\psi_F^\mu(v,\xi)} \leq \theta_{k,1}(v,\xi) + \theta_{k,2}(v,\xi),
  \]
  with
  \begin{align*}
    \theta_{k,1}(v,\xi) &= \sup_{F\in B}\abs[\big]{\tilde r_{\mu_k}(F,A_v)-\tilde r_{\mu}(F,A_v)}\sup_{F\in B}\abs[\Big]{\frac{\rmd Q_F^{\mu_k}}{\rmd Q_0}(\xi)}\\
    \text{and }\theta_{k,2}(v,\xi) &=\sup_{F\in B}\abs[\big]{\tilde r_{\mu}(F,A_v)}\sup_{F\in B}\abs[\Big]{\frac{\rmd Q_F^{\mu_k}}{\rmd Q_0}(\xi)-\frac{\rmd Q_F^\mu}{\rmd Q_0}(\xi)}.
  \end{align*}
  We show that, for fixed $v$, $\sup_{k\geq 1}(\theta_{k,1}+\theta_{k,2})\in L^q(Q_0)$ and $(\theta_{k,1}+\theta_{k,2})\to 0$, $Q_0$-a.s. By dominated convergence, we have then $\norm{\theta_{k,1}+\theta_{k,2}}_{L^q(Q_0)}\to 0$ and the result follows.  By \Cref{lem:2-prop-ibo-lipschitz}, $\rmd Q_F^{\mu_k}/\rmd Q_0(\xi)$ is uniformly bounded by $K^{2^d-1}$. Assumption~(iii) then implies that $\sup_{k\geq 1}(\theta_{k,1}+\theta_{k,2})\in L^q(Q_0)$. On the other hand, Assumption~(i) ensures that  $\theta_{k,1} \to 0$ $Q_0$-a.s. and Assumption (ii) together with Equations \eqref{eq:RN2}-\eqref{eq:Delta_i-expression} ensure that $\theta_{k,2} \to 0$ $Q_0$-a.s.
\end{proof}

We can now use the result above to prove Proposition~\ref{prop:continuity_ibo_igb}.

\begin{proof}[Proof of Proposition~\ref{prop:continuity_ibo_igb} (i)]
Let $n\geq 1$ be fixed, consider an input sample $(\mathbf{x},\mathbf{y})=(x_i,y_i)_{1\leq i\leq n}$ and define $\mu=\sum_{i=1}^n \delta_{(x_i,y_i)}$.  For a sequence of input samples $(\mathbf{x}^k,\mathbf{y}^k)=(x^k_i,y^k_i)_{1\leq i\leq n}$ that tends to $(\mathbf{x},\mathbf{y})$ as $k\to \infty$, we write $\mu_k=\sum_{i=1}^n \delta_{(x_i^k,y_i^k)}$ for the empirical distribution associated with $(\mathbf{x}^k,\mathbf{y}^k)$. It is easily checked that $\mu,\mu_k$ satisfies \Cref{ass:L-regular-convex,ass:Y-integrable,ass:bound-ratios,ass:bounded-lipshitz} (with any $q>1$ in \Cref{ass:Y-integrable}).
We want to show that for any $q\geq 1$ and any bounded subset $B\subset \mathbb{B}$, we have
\[
  \sup_{F\in B}\norm{\mathcal{T}_{\mu_k}(F)-\mathcal{T}_\mu(F)}_{\mathbb{W}^q} \underset{k\to\infty}{\longrightarrow} 0.
\]
This is easily proven thanks to \Cref{lem:deterministic-T-cv}. The assumptions \textit{(i)--(iii)} of the lemma are easily verified.
Point \textit{(iii)} is satisfied because necessarily $\{(\mathbf{x},\mathbf{y})\}\cup\big(\bigcup_{k\geq 1}\{(\mathbf{x}^k,\mathbf{y}^k)\}\big)$ is a compact subset of $[0,1]^{pk}\times \mathbb{R}^k$, so by a continuity argument, we have
\begin{align*}
  &\sup_{F\in B}\sup_{A\subset [0,1]^p}\sup_{k\geq 1} \,\abs{\tilde{r}_{\mu_k}(F, A)} \\
 &\quad =\sup_{F\in B}\sup_{A\subset [0,1]^p}\sup_{k\geq 1} \,\abs[\Big]{\frac{\sum_{i}\pderiv{z}L(y_i^k,F(x_i^k))\1_{A}(x_i^k)}{\sum_{i}\pderiv[2]{z}L(y_i^k,F(x_i^k))\1_{A}(x_i^k)}} < \infty.
\end{align*}
For \textit{(i)--(ii)}, it is clear that the convergence holds for splitting schemes $\xi$ such that none of the $(x_i)_{1\leq i\leq n}$ is at the frontier of a leaf.
This event has null $Q_0$-probability since the splits of a  completely random splitting scheme are uniform.
Therefore the lemma applies, concluding the proof.
\end{proof}

\begin{proof}[Proof of Proposition~\ref{prop:continuity_ibo_igb} (ii)]
 We write respectively $(\hat F_t^n)_{t\geq 0}$ and $(\hat F_t^n{}')_{t\geq 0}$ for the infinitesimal gradient boosting process based on the input samples $(\mathbf{x},\mathbf{y})$ and $(\mathbf{x}',\mathbf{y}')$. Since $(\hat F_t^n)_{t\geq 0}$ is the solution of the ODE, 
\[
\hat F_t^n=\hat F_0^n+\int_0^t \mathcal{T}_n(\hat F_u^n)\,\rmd u,\quad t\geq 0,
\]
and similarly for $\hat F_t^n{}'$. We deduce
\begin{equation}\label{eq:gronwall}
\norm[\big]{\hat F_t^n{}'-\hat F_t^n}_{\mathbb{W}^q}\leq \norm[\big]{\hat F_0^n{}'-\hat F_0^n}_{\mathbb{W}^q}+\int_0^t \norm[\big]{\mathcal{T}_n'(\hat F_u^n{}')-\mathcal{T}_n(\hat F_u^n)}_{\mathbb{W}^q}\,\rmd u.
\end{equation}
 By the triangle inequality, 
\[
\norm[\big]{\mathcal{T}_n'(\hat F_u^n{}')-\mathcal{T}_n(\hat F_u^n)}_{\mathbb{W}^q}
\leq  \norm[\big]{\mathcal{T}_n'(\hat F_u^n{}')-\mathcal{T}_n(\hat F_u^n{}')}_{\mathbb{W}^q}+\norm[\big]{\mathcal{T}_n(\hat F_u^n{}')-\mathcal{T}_n(\hat F_u^n)}_{\mathbb{W}^q}.
\]
Let the time horizon $T>0$ be fixed  and let $B_T\subset \mathbb{W}^q$ be a bounded set containing $\hat F_u^n$ and $\hat F_u^n{}'$ for $u\in[0,T]$ and $C_T$ the Lipschitz constant on $B_T$ of the locally Lipschitz map $\mathcal{T}^n$. We have
\begin{align*}
 \norm[\big]{\mathcal{T}_n'(\hat F_u^n{}')-\mathcal{T}_n(\hat F_u^n{}')}_{\mathbb{W}^q}&\leq K := \sup_{F\in B_T}\norm[\big]{\mathcal{T}_n'(F)-\mathcal{T}_n(F)}_{\mathbb{W}^q},\\
\norm[\big]{\mathcal{T}_n(\hat F_u^n{}')-\mathcal{T}_n(\hat F_u^n)}_{\mathbb{W}^q} &\leq C_T \norm[\big]{\hat F_u^n{}'-\hat F_u^n}_{\mathbb{W}^q}.
\end{align*}
These bounds together with Equation~\eqref{eq:gronwall} imply
\[
\norm[\big]{\hat F_t^n{}'-\hat F_t^n}_{\mathbb{W}^q}\leq  \norm[\big]{\hat F_0^n{}'-\hat F_0^n}_{\mathbb{W}^q}+Kt+C_T\int_0^t \norm[\big]{\hat F_u^n{}'-\hat F_u^n}_{\mathbb{W}^q}\,\rmd u.
\]
Grönwall's Lemma finally yields
\[
\norm[\big]{\hat F_t^n{}'-\hat F_t^n}_{\mathbb{W}^q}\leq \big( \norm[\big]{\hat F_0^n{}'-\hat F_0^n}_{\mathbb{W}^q}+Kt\big)e^{C_Tt},\quad t\in[0,T].
\]
By point (i) proven above,  $K= \sup_{F\in B_T}\norm[\big]{\mathcal{T}_n'(F)-\mathcal{T}_n(F)}_{\mathbb{W}^q}\to 0$ as $(\mathbf{x}',\mathbf{y}')\to(\mathbf{x},\mathbf{y})$ and therefore it remains to show that $\norm{\hat F_0^n{}'-\hat F_0^n}_{\mathbb{W}^q}\to 0$ as $(\mathbf{x}',\mathbf{y}')\to(\mathbf{x},\mathbf{y})$. Recall the initialization is constant and given by
$\hat F_0^{n}{}'=\argmin_{z} \sum_{i=1}^n L(y_i',z)$. Therefore, in a neighborhood of $\mathbf{y}$, the implicit function theorem implies the continuity of the map
\[
  \mathbf{y}' \mapsto \Big(\text{unique solution of }\sum_{i=1}^n \pderiv{z}L(y_i',z)=0\Big).
\]
Note that the theorem can be applied  since, by \Cref{ass:L-regular-convex},  $L$ is $C^2$ with $\pderiv[2]zL>0$.  We deduce that $\norm{\hat F_0^n{}'-\hat F_0^n}_{\mathbb{W}^q}=\mathrm{cst}\times |\hat F_0^n{}'-\hat F_0^n|\to 0$, proving the result.
\end{proof}

\subsubsection{Proof of Theorem~\ref{thm:asymptotic_ibo}}

\begin{proof}[Proof of \Cref{thm:asymptotic_ibo}]
We aim to use \Cref{lem:deterministic-T-cv}, and need to check that the conditions \textit{(i)--(iii)} of the lemma hold almost surely for $\mu_k = \P_k$ and $\mu=\P$.

We first show (i): $\displaystyle \sup_{F\in B}\sum_{v\in\{0,1\}^d} \abs[\big]{\tilde{r}_n(F,A_v)-\tilde{r}(F,A_v)} \to 0$, $Q_0$-a.s.\\
Note that for any hyperrectangle $A$ such that $\P(x\in A)>0$, \Cref{prop:GC} ensures the almost sure convergence $\tilde{r}_n(F,A_v)\to \tilde{r}(F,A_v)$.
Furthermore, for any $A$ such that $\P(x\in A)=0$, we have $\tilde{r}_n(F,A_v)=0$ and $\tilde{r}(F,A_v)=0$ almost surely.
Therefore (i) is satisfied.

The proof of (ii), which consists in showing that 
\[
\sup_{F\in B}\sum_{v\in\{0,1\}^d} \abs[\big]{\frac{\P_n[\pderiv{z}L(y,F(x))\1_{A_v}(x)]^2}{\P_n(x\in A_v)}-\frac{\P[\pderiv{z}L(y,F(x))\1_{A_v}(x)]^2}{\P(x\in A_v)}} \to 0, \quad Q_0\text{-a.s.},
\]
is exactly the same.

We now show (iii): $\displaystyle\int \sup_{n\in\mathbb{N},F\in B}\sum_{v\in\{0,1\}^d}  \abs[\big]{\tilde{r}_n(F,A_v)}^q \,Q_0(\rmd \xi) < \infty$.\\
This trivially holds under \Cref{ass:bound-ratios}, (i) because it implies that the ratios $\tilde{r}_n(F,A)$ are uniformly bounded.
Under \Cref{ass:bound-ratios}, (ii), instead, we can define $\delta = \inf_{(y,z)\in \mathcal{Y}\times K}\pderiv[2]{z}L(y,z)>0$, where $K=[-M,M]$, with $M=\sup_{F\in B}\norm{F}_\infty$.
We can then bound
\[
  \sup_{F\in B,n\geq 1}\tilde{r}_n(F,A) \leq \frac{1}{\delta}\sup_{n\geq 1} \frac{\P_n[\sup_{z\in K}\abs{\pderiv{z}L(y,z)}\1_{A_v}(x)]}{\P_n(x\in A)}.
\]
Note that the supremum in the right hand side is equal in distribution to
\[
\sup_{n\geq 1}\frac{1}{n}\sum_{i=1}^n g(\tilde{Y}_i),
\]
where $g(y)= \sup_{z\in K}|\pderiv{z}{L}(y,z)|$, and the $(\tilde{Y}_i)_{i\geq 1}$ are i.i.d.\ with the conditional distribution of $Y$ given $x\in A$.
It is classical \cite[see e.g.][Example 5.6.1]{Dur10} that $(\frac{1}{n}\sum_{i=1}^n g(\tilde{Y}_i))_{n\geq 1}$ is a backwards martingale, and Doob's inequality \cite[Theorem 5.4.3]{Dur10} implies
\[
  \E\left[\Big(\sup_{n\geq 1}\frac{1}{n}\sum_{i=1}^n g(\tilde{Y}_i)\Big)^q\right] \leq \left(\frac{q}{q-1}\right)^q \E[g(\tilde{Y})^q].
\]
Now by \Cref{ass:Y-integrable}, we can bound $\E[g(\tilde{Y})^q]$ by a constant $C$ that does not depend on $A$, so we have
\[
  \E\big[\sup_{F\in B,n\geq 1}\tilde{r}_n(F,A)^q\big] \leq \left(\frac{q}{q-1}\right)^qC < \infty.
\]
Taking the integral with respect to $\xi$, this shows that almost surely, (iii) is satisfied, concluding the proof.
\end{proof}

\subsection{Proof of Proposition \ref{prop:ODE} and of Lemma \ref{lem:tmax}}

\Cref{prop:ODE} is immediate from the fact that $\mathcal{T}:\mathbb{W}^\infty \to \mathbb{W}^\infty$ is locally Lipschitz, which is a straightforward consequence of \Cref{prop:ibo_Lipschitz}.

\begin{proof}[Proof of \Cref{lem:tmax}]
  By providing control on the leaf values of a softmax gradient tree, Equation \eqref{eq:infinite-tmax} ensures that for $F\in \mathbb{B}$, we have $\norm{\mathcal{T}(F)}_{\mathbb{W}^\infty} \leq A \norm{F}_{\infty} + B$.
  Therefore, from any initial condition $F_0\in\mathbb{B}$, the solution to
  \[
    \deriv{t} F_t = \mathcal{T}(F_t)
  \]
  satisfies
  \[
    \norm{F_t - F_0}_{\infty} \leq \norm{F_t - F_0}_{\mathbb{W}^\infty} \leq \int_0^t \norm{\mathcal{T}(F_s)}_{\mathbb{W}^\infty} \, \rmd s \leq \int_0^t A\norm{F_s}_{\infty} + B \, \rmd s.
  \]
  A standard Grönwall lemma-type argument shows that the norm of $F_t$, hence the norm $\mathcal{T}(F_t)$, cannot explode in finite time, therefore the maximal time of definition of $F_t$ is $t_{\max}=+\infty$.
\end{proof}

\subsection{Proof of Theorem~\ref{thm:asymptotic_igb}}
\begin{proof}[Proof of Theorem~\ref{thm:asymptotic_igb}] The proof shares similarities with the proof of Proposition~\ref{prop:continuity_ibo_igb} (ii) and relies on Grönwall's Lemma.

First let us show that $\hat{F}^n_0 \to \hat{F}_0$ almost surely.  According to Assumption~\ref{ass:L-regular-convex}, the  map $z\mapsto \mathbb{E}[L(Y,z)]$ has a unique minimizer  $\hat F_0$ which must be the unique zero of the derivative $z\mapsto \mathbb{E}[\partial L(Y,z)]$ --- note that Assumption~\ref{ass:Y-integrable} ensures that one can differentiate under the expectation. The maps $z\mapsto P_n[L(y,z)]$, $n\geq 1$, are strictly convex and, by the law of large numbers,  their derivatives satisfy, for all $\epsilon>0$,
\[
P_n\big[\pderiv{z} L(y,\hat F_0-\epsilon)\big]\longrightarrow \mathbb{E}\big[\pderiv{z} L(Y,\hat F_0-\epsilon)\big]>0 \quad \mbox{a.s.}
 \]
 and
 \[ P_n\big[\pderiv{z} L(y,\hat F_0+\epsilon)\big]\longrightarrow \mathbb{E}\big[\pderiv{z} L(Y,\hat F_0-\epsilon)\big]>0\quad \mbox{a.s.}
\]
This ensures that, for $n$ large enough,  $\hat F_0^n\in [\hat F_0-\epsilon,\hat F_0+\epsilon]$ and,  $\epsilon>0$ being arbitrary, proves the almost sure convergence  $\hat{F}^n_0 \to \hat{F}_0$.

Now, note that for all $t\in [0,t_{\max})$, 
\[
\hat F_t=\hat F_0+\int_0^t \mathcal{T}(\hat F_u)\,\rmd u\quad
\mbox{and}\quad 
\hat F_t^n=\hat F_0^n+\int_0^t \mathcal{T}_n(\hat F_u^n)\,\rmd u.
\]
Taking the difference and applying the triangle inequality, we deduce
\begin{equation} \label{eq:preGronwall}
\norm[\big]{\hat F_t^n-\hat F_t}_{\mathbb{W}^q}\leq \norm[\big]{\hat F_0^n-\hat F_0}_{\mathbb{W}^q}+\int_0^t \norm[\big]{\mathcal{T}_n(\hat F_u^n)-\mathcal{T}(\hat F_u)}_{\mathbb{W}^q}\,\rmd u.
\end{equation}
The integrand in the right hand side is bounded from above by
\begin{equation} \label{eq:GronwallTriangleIneq}
\norm[\big]{\mathcal{T}_n(\hat F_u^n)-\mathcal{T}(\hat F_u)}_{\mathbb{W}^q}
\leq  \norm[\big]{\mathcal{T}_n(\hat F_u^n)-\mathcal{T}(\hat F_u^n)}_{\mathbb{W}^q}+\norm[\big]{\mathcal{T}(\hat F_u^n)-\mathcal{T}(\hat F_u)}_{\mathbb{W}^q}.
\end{equation}
Let us now fix a time horizon $T\in (0,t_{\max})$ and show that $\sup_{t\in [0,T]}\norm{\hat{F}_t^n-\hat{F}_t}_{\mathbb{W}^q} \to 0$ almost surely.
Define
\[
R=\sup_{t\in [0,T]} \norm{\hat{F}_t}_{\mathbb{W}^q} \quad \text{and}\quad M = \sup_{F\in B(R+1)}\norm{\mathcal{T}(F)}_{\mathbb{W}^q},
\]
where $B(R+1)$ denotes the closed ball of $\mathbb{W}^q$ centered on $0$ and of radius $R+1$.
For $n\geq 1$, define
\[ S_n=\max\{s\in [0,T]\ :\ \forall m\geq n, \,\hat{F}^n_s\in B(R+1)\}. \]
Note that be definition the $(S_n)_{n\geq 1}$ are nondecreasing.
We will show that
\begin{enumerate}[label=(\roman*)]
  \item $\sup_{t\in [0,S_n]} \sup_{m\geq n} \norm[\big]{\hat F_t^m-\hat F_t}_{\mathbb{W}^q} \to 0$.
  \item for all $n\geq 1$, there almost surely exists $n' \geq n$ such that $S_{n'} \geq (S_n+\delta_M)\wedge T$, for $\delta_M={(2(M+1))^{-1}}$.
\end{enumerate}
Since the second point implies that there almost surely exists $n\geq 1$ such that $S_n=T$, by the first point we will have
\[
  \sup_{t\in [0,T]}\norm[\big]{\hat F_t^n-\hat F_t}_{\mathbb{W}^q} \;\underset{n\to\infty}{\longrightarrow}\; 0\qquad a.s,
\]
which proves the proposition.

Let us now show (i) and (ii).
Let $C$ denote the Lipschitz constant on $B(R+1)$ of the locally Lipschitz map $\mathcal{T}$, and for $n\geq 1$, let us define
\[
  \epsilon_{n} = \sup_{F\in B(R+1)} \sup_{m\geq n} \norm[\big]{ \mathcal{T}_m(F)-\mathcal{T}(F)}_{\mathbb{W}^q},
\]
which tends to $0$ almost surely by \Cref{thm:asymptotic_ibo}.
By \eqref{eq:preGronwall} and \eqref{eq:GronwallTriangleIneq}, for $t\in [0,S_n]$, we have
\[
\norm[\big]{\hat F_t^n-\hat F_t}_{\mathbb{W}^q}\leq \norm[\big]{\hat F_0^n-\hat F_0}_{\mathbb{W}^q} + \epsilon_{n}T + \int_0^t C \norm[\big]{\hat F_u^n-\hat F_u}_{\mathbb{W}^q}\,\rmd u,
\]
whence Grönwall's Lemma implies
\[
\norm[\big]{\hat F_t^n-\hat F_t}_{\mathbb{W}^q}\leq \big( \norm[\big]{\hat F_0^n-\hat F_0}_{\mathbb{W}^q}+\epsilon_{n}T\big)e^{Ct}.
\]
Therefore we have
\[
\sup_{t\in[0,S_n]}\sup_{m\geq n}\norm[\big]{\hat F_t^m-\hat F_t}_{\mathbb{W}^q}\leq \big( \sup_{m\geq n}\norm[\big]{\hat F_0^m-\hat F_0}_{\mathbb{W}^q}+\epsilon_{n}T\big)e^{CT} \;\underset{n\to\infty}{\longrightarrow}\; 0\qquad a.s.
\]
This shows (i). To prove (ii), note that for any fixed $n\geq 1$ there is a random index $n'\geq n$ such that
\[
  \sup_{t\in [0,S_n]} \sup_{m\geq n'}\norm{\hat{F}^m_t-\hat{F}_t}_{\mathbb{W}^q} \leq \frac{1}{2} \quad\text{ and } \quad\sup_{F\in B(R+1)} \sup_{m\geq n'} \norm{\mathcal{T}^m(F)}_{\mathbb{W}^q} \leq M+1.
\]
This implies that for all $m\geq n'$, $\norm{\hat{F}^m_{S_n}}_{\mathbb{W}^q} \leq R+\frac{1}{2}$, and furthermore that for all $t\in [S_n,S_n+\delta_M]$, we have $\norm{\hat{F}^m_t-\hat{F}^m_{S_n}}_{\mathbb{W}^q} \leq \frac{1}{2}$.
Therefore we have $\hat{F}^m_t\in B(R+1)$ for all $m\geq n'$ and all $t\in [0,S_n+\delta]$, in other words $S_{n'} \geq (S_n+\delta)\wedge T$.
So we have shown (ii) and this completes the proof.
\end{proof}

\subsection{Proof of Proposition~\ref{prop:decreasing_error} and Proposition~\ref{prop:zero_residual}}

\begin{proof}[Proof of \Cref{prop:decreasing_error}]
We differentiate $\mathbb{E}[L(Y,\hat{F}_t(X))]$ with respect to time.
To see that we can differentiate under the expectation, note that for a fixed $t\in(0,t_{\mathrm{max}})$ and any $s\in [0,t]$, we have
\[
\abs[\Big]{\deriv{s}L(Y,\hat{F}_s(X))} = \abs[\big]{\mathcal{T}(X;\hat{F}_s)\pderiv{z}L(Y,\hat{F}_s(X))} \leq M \sup_{z\in K}\abs[\Big]{\pderiv{z}L(Y,z)},
\]
where $M=\sup_{t\in[0,T],x\in [0,1]^p}\mathcal{T}(x;\hat{F}_s)$ and $K$ is the image of $[0,T]\times[0,1]^p$ by $(t,x)\mapsto \hat{F}_s(x)$.
The right-hand side expression does not depend on $t$ and is integrable by \Cref{ass:Y-integrable}, therefore we can differentiate under the expectation and compute
\begin{align}
  \deriv{t}\E[L(Y,\hat{F}_t(X))] &= \E\left[\mathcal{T}(X;\hat{F}_t)\pderiv{z}L(Y,\hat{F}_t(X))\right] \notag\\
  &= -\E\left[\sum_{v\in \{0,1\}^d} \frac{\P[\pderiv{z}L(y,\hat{F}_t(x))\1_{A_v(x)}]}{\P[\pderiv[2]{z}L(y,\hat{F}_t(x))\1_{A_v(x)}]}\1_{A_v}(X)\pderiv{z}L(Y,\hat{F}_t(X))\right] \notag\\
  &= -\E\left[\sum_{v\in \{0,1\}^d} \frac{\P[\pderiv{z}L(y,\hat{F}_t(x))\1_{A_v(x)}]^2}{\P[\pderiv[2]{z}L(y,\hat{F}_t(x))\1_{A_v(x)}]}\right] \leq 0, \label{eq:time-diff-of-error}
\end{align}
where the $(A_v)$ are the leaves of a regression tree $\xi$ based on $\hat{F}_t$,
which concludes the proof.
\end{proof}

\begin{proof}[Proof of \Cref{prop:zero_residual}]
We leave the computation --- which is very similar to the one in the proof above --- to the reader; we get
\[
  \deriv{t}\,\E\!\left[\pderiv{z}L(Y,\hat{F}_t(X))\right] = - \E\left[\pderiv{z}L(Y,\hat{F}_t(X))\right],
\]
which implies the result.
\end{proof}

\subsection{Proof of Proposition \ref{prop:critical-points} and Proposition \ref{prop:weak-consistency}}

\begin{proof}[Proof of \Cref{prop:critical-points}]
First, note that (ii) clearly implies (i) because (ii) implies that leaf-values of a softmax gradient tree are always null.

(i)$\implies$(ii). If $\mathcal{T}(F)=0$, then $\E[\pderiv{z}L(Y,F(X))\mathcal{T}(X; F)] = 0$, and this expectation was computed in \eqref{eq:time-diff-of-error} to be equal to
\[
  -\E\left[\sum_{v\in \{0,1\}^d} \frac{\P[\pderiv{z}L(y,\hat{F}_t(x))\1_{A_v(x)}]^2}{\P[\pderiv[2]{z}L(y,\hat{F}_t(x))\1_{A_v(x)}]}\right].
\]
This shows that for $Q_0$-almost every splitting scheme $\xi$ and all $v\in \{0,1\}^d$, the non-negative value $\P[\pderiv{z}L(y,\hat{F}_t(x))\1_{A_v(x)}]^2$ must be zero.

Since under $Q_0$, the splitting scheme $\xi$ is completely random, i.e.\ the regions $A_v$ are obtained by successively making a series of $d$ uniform splits.
Therefore we have $\P[\pderiv{z}L(y,\hat{F}_t(x))\1_{A(x)}]=0$ for a dense subset $A\in \mathcal{A}_d$, where $\mathcal{A}_d$ is endowed with the metric $d(A,A'):=\leb(A\symdiff A')$ --- recall that $\leb$ denotes the Lebesgue measure and $\symdiff$ the symmetric difference.
A standard continuity argument shows that it must also hold for all $A\in \mathcal{A}_d$.

(ii)$\iff$(iii). Let $J\subset \{1,\dots,p\}$ with $\abs{J}\leq d$, and define $\mathcal{A}_J\subset \mathcal{A}_d$ as the sets $A$ of the form $[a,b]$ for which $(a_j, b_j)=(0,1)$ for all $j\notin J$.
Clearly the $\sigma$-field generated by $\mathcal{A}_J$ is the one that makes the map $x\in[0,1]^p\mapsto x_J$ measurable.
Therefore if $\E\left[\pderiv{z}L(Y,F(X))\1_{A}(X)\right]=0$ for all $A\in \mathcal{A}_J$, then $\E\left[\pderiv{z}L(Y,F(X)) \;\Big |\; X_J\right] = 0$ almost surely, and reciprocally.
\end{proof}

In the following we focus on the context of regression.
Let us recall that we consider $L^2=L^2([0,1]^p,\P_X)$, where $\P_X$ denotes the distribution of $X$ under $\P$, endowed with its usual scalar product $\langle\cdot,\cdot\rangle$.
Recall also that we define $L^2_d=\overline{\mathrm{span}}(\1_{A},\ A\in\mathcal{A}_d)$.

The following lemma is key to proving \Cref{prop:weak-consistency}, which shows a convergence in the weak topology of $L^2$ --- i.e.\ the coarsest topology for which for each $g\in L^2$, the map $f\mapsto \langle f,g\rangle$ is continuous.
Recall that we defined $F^*= \E[Y\mid X]\in L^2$ --- this is the so-called target function of the problem of regression.

\begin{lemma} \label{lem:weak-cont-calt}
  The map
  \begin{equation*} 
  \phi : F \in L^2 \mapsto \int \sum_{v\in \{0,1\}^d} \langle F^*-F, \1_{A_v}\rangle^2\frac{\rmd Q_F}{\rmd Q_0}(\xi)\,Q_0(\rmd \xi).
  \end{equation*}
  is weakly continuous on bounded subsets of $L^2$, and is null only on the affine space $F^*+(L^2_d)^{\perp}$.
\end{lemma}
\begin{proof}
The second part of the lemma is easily shown: similarly as in the proof of \Cref{prop:critical-points} above, it is clear that $\phi(F)=0$ if and only if $\langle F^*-F,\1_{A}\rangle=0$ for all $A\in \mathcal{A}_d$, and this is equivalent to $F\in F^*+(L^2_d)^{\perp}$.
  
To prove the first part, observe that the functions $\1_{A}$ are in the unit ball of $L^2$, for any Borel subset $A\subset [0,1]^p$.
Also, in the case of regression, note that for a Borel subset $A\subset [0,1]^p$, we have
\[
  \E[\textstyle \pderiv{z}L(Y,F(X))\1_{A}(X)] =\E[(Y-F(X))\1_{A}(X)] = \langle F^*-F,\1_{A} \rangle,
\]
with $F^*= \E[Y\mid X]\in L^2$; furthermore, $\P(X\in A) = \norm{\1_{A}}^2_{L^2}$.
Therefore, the scores~\eqref{eq:Delta_i-expression} used in \Cref{algo1} can be written
\begin{align*}
\Delta(j,u,A) &= \frac{\langle F-F^*, \1_{A_{0}}\rangle^2}{\norm{\1_{A_{0}}}^2_{L^2}}+\frac{\langle F-F^*, \1_{A_{1}}\rangle^2}{\norm{\1_{A_{1}}}^2_{L^2}} \\
& = \langle F-F^*, \widetilde{\1}_{A_{0}}\rangle^2+\langle F-F^*, \widetilde{\1}_{A_{1}}\rangle^2,
\end{align*}
where the regions $A_{0}$ and $A_{1}$ are the result of the $(j,u)$-split of the region $A$, and where $\widetilde{\1}_A$ denotes the normalization of $\1_A$ in $L^2$.
Expressing $\frac{\rmd Q_F}{\rmd Q_0}(\xi)$ in terms of the scores, as in~\eqref{eq:RN2} with $\mu=\P$, it is clear that the map $\phi$ is of the form
\[
  \phi(F) = \int_{B^k}\psi((\langle F-F^*,g_i\rangle)_{1\leq i\leq k})\,\mu(\rmd g_1 \dots \rmd g_k),
\]
for some $k\geq 1$, where $\mu$ is a probability measure on $B^k$ and $\psi:\mathbb{R}^k\to\mathbb{R}$ is continuous.
Let us fix a bounded domain $D$ of $L^2$ --- recall that we are interested in showing that $\phi$ is weakly continuous on bounded subsets of $L^2$ --- and consider our map $\phi$ as a function of $F\in D$.
We then have almost surely $\langle F^*-F,g_i\rangle\in [-C,C]$ for all $i$ for some constant $C>0$, so we can assume without loss of generality that $\psi$ is bounded and uniformly continuous, with a modulus of uniform continuity defined by
\[
  \omega_\psi(\epsilon) = \sup\big\{ \abs{\psi(x)-\psi(y)}\ :\ x, y \in [-C,C]^p \text{ with }\norm{x-y}_\infty \leq \epsilon \big\}.
\]

Let us now fix $\epsilon > 0$, and consider a finite set $G=\{g_1,\dots,g_m\}\subset B$ such that $\mu((G^{\epsilon})^k) > 1-\epsilon$, where
\[
  G^{\epsilon} = \bigcup_{i=1}^mB(g_i,\epsilon).
\]
Note that we can do this since $B$ is a Polish (metric and complete) space.
Now for any fixed $F\in D$, consider a weak neighborhood $V$ of $F$ such that for all $F'\in V$, for all $i\in \{1,\dots,m\}$, $\abs{\langle F-F',g_i\rangle} < \epsilon$.
Then for all $F'\in V\cap D$ and $g\in G^{\epsilon}\cap B$, we have
\[
  \abs{\langle F-F',g\rangle} \leq  3\epsilon,
\]
so that for all $F'\in V\cap D$,
\begin{align*}
  &\int_{B^k}\abs{\psi((\langle F^*-F,g_i\rangle)_{1\leq i\leq k})-\psi((\langle F^*-F',g_i\rangle)_{1\leq i\leq k})}\,\mu(\rmd g)\\
  &\quad \leq\omega_\psi(3\epsilon)\mu((G^{\epsilon})^k) + 2\norm{\psi}_\infty (1-\mu((G^\epsilon)^k)) \\
  &\quad \leq \omega_\psi(3\epsilon) + 2\epsilon \norm{\psi}_\infty,
\end{align*}
This tends to $0$ as $\epsilon\to 0$, so we have proved that $\phi$ is weakly continuous on bounded subsets of $L^2$, and the proof is complete.
\end{proof}

We state a final lemma before being able to prove \Cref{prop:weak-consistency}; its second part is an interesting observation in itself and it is convenient to prove it here, but it is the first part that will be useful for the following proof.

\begin{lemma} \label{lem:calt-F-bounded}
  In the context of regression, for all $F\in L^2$,
  \[
  \norm{\mathcal{T}(F)}_{L^2} \leq 2^dK^{2^d-1}\norm{F-F^*}_{L^2}.
  \]
  In the Adaboost setting, for all $F\in \mathbb{B}$,
  \[
  \norm{\mathcal{T}(F)}_\infty \leq 2^dK^{2^d-1}.
  \]
\end{lemma}
\begin{proof}
  For the first part of the lemma, note that we can write
  \[
  \mathcal{T}(F) = \int \sum_{v\in\{0,1\}^d} \frac{\langle F^*-F,\1_{A_v} \rangle}{\norm{\1_{A_v}}_{L^2}^2} \1_{A_v}(\cdot) \frac{\rmd Q_F}{\rmd Q_0}(\xi) \,Q_0(\rmd \xi).
  \]
  Since we have $\frac{\rmd Q_F}{\rmd Q_0}(\xi)\leq K^{2^d-1}$, for any $g\in L^2$ we can bound
  \begin{align*}
  \abs{\langle g, \mathcal{T}(F)\rangle} &\leq K^{2^d-1}\int \sum_{v\in\{0,1\}^d} \frac{\abs{\langle F^*-F,\1_{A_v} \rangle}}{\norm{\1_{A_v}}_{L^2}} \frac{\abs{\langle g,\1_{A_v} \rangle}}{\norm{\1_{A_v}}_{L^2}}\,Q_0(\rmd \xi)\\
  &\leq 2^dK^{2^d-1}\norm{F^*-F}_{L^2} \norm{g}_{L^2}.
  \end{align*}
  Taking the supremum of this expression for $g$ in the unit ball of $L^2$ yields the result.
  
  For the second part of the lemma, note that in the context of Adaboost, we have
  \[
  \tilde{r}(F,A)= \frac{\E[Ye^{-YF(X)}\1_{A}(X)]}{\E[e^{-YF(X)}\1_{A}(X)]}\in [-1;1] \quad \text{because }Y\in \{-1,1\}\quad \text{a.s}.
  \]
  Therefore, similarly as above we can directly bound $\norm{\mathcal{T}(F)}_{\infty} \leq 2^d K^{2^d-1}$.
\end{proof}

\begin{proof}[Proof of \Cref{prop:weak-consistency}, (i)]
By \Cref{prop:decreasing_error}, $\norm{\hat{F}_t-F^*}_{L^2}$ is decreasing in time, so we can fix $B\subset L^2$ a closed ball such that $\hat{F}_t\in B$ for all $t\geq 0$.
Since closed balls in $L^2$ are compact for the weak topology, we know that $\hat{F}_t$ has weakly convergent subsequences.

It is sufficient to show that any such subsequence must tend to $\mathcal{P}_d(F^*)$.
Let us fix $\ell$ an adherent point.
We know that $\ell$ must be in $B$ because the norm is weakly lower semicontinuous, and furthermore it is readily seen that $L^2_d$ is weakly closed, so that $\ell \in L^2_d\cap B$.

We assume by contradiction that $\ell \neq \mathcal{P}_d(F^*)$, and compute
\[
\deriv{t}\norm{\hat{F}_t-F^*}_{L^2}^2 = 2\langle \hat{F}_t-\mathcal{P}_d(F^*), \mathcal{T}(\hat{F}_t)\rangle = -\int \sum_v \langle \hat{F}_t-\mathcal{P}_d(F^*), \1_{A_v}\rangle^2\frac{\rmd Q_{\hat{F}_t}}{\rmd Q_0}(\xi)\,Q_0(\rmd \xi).
\]
By \Cref{lem:weak-cont-calt}, the map
\[
\phi:F \in L^2_d\cap B \mapsto \int \sum_v \langle F-\mathcal{P}_d(F^*), \1_{A_v}\rangle^2\frac{\rmd Q_F}{\rmd Q_0}(\xi)\,Q_0(\rmd \xi).
\]
is continuous for the weak topology, nonnegative, and null only when $F = \mathcal{P}_d(F^*)$, therefore $c:=\phi(\ell)$ is positive and there exists $V$ a weak neighborhood of $\ell$ such that for all $f\in V$, $\phi(f)>c/2$.
Up to taking a smaller neighborhood, we can assume that $V$ is of the form
\[
  V = \big\{f\in L^2_d\cap B, \, \forall i\in \{1,\dots k\}, \langle f-\ell, g_i \rangle < \epsilon \big\}
\]
for some $g_1,\dots g_k$ in the unit ball of $L^2$ and $\epsilon > 0$.
Let us now define $W \subset V$ by
\[
  W = \big\{f\in L^2_d\cap B, \, \forall i\in \{1,\dots k\}, \langle f-\ell, g_i \rangle < \tfrac{\epsilon}{2} \big\}.
\]
Consider the set of times $T_V = \{t\geq 0, \,\hat{F}_t \in V\}$, and define $T_W$ in an analogous way.
Since $V$ and $W$ are both neighborhoods of $\ell$, these sets are unbounded.
Furthermore, if $t\in T_W$, then for each $u\geq 0$, for all $i\in \{1,\dots,k\}$, we have
\[
  \abs{\langle \hat{F}_t - \ell, g_i \rangle -\langle \hat{F}_u - \ell, g_i \rangle} \leq \abs{t-u} C
\]
for $C:= \sup_{F\in B}\norm{\mathcal{T}(F)}_{L^2}$, which is finite by \Cref{lem:calt-F-bounded}.
Therefore, for all $t\in T_W$, if $\abs{t-u} \leq \frac{\epsilon}{2C}$, then $u\in T_V$, in other words $[t-\frac{\epsilon}{2C},t+\frac{\epsilon}{2C}]\subset T_V$.
Since $T_W$ is unbounded, this shows that the total time spent in $V$ is infinite.
This is absurd since we would have
\[
\infty > \lim_{t\to\infty} \norm{\hat{F}_0-F^*}_{L^2}^2-\norm{\hat{F}_t-F^*}_{L^2}^2=\int_0^{\infty}\phi(\hat{F}_t)\,\rmd t > \int_{T_V} \frac{c}{2} \, \rmd t = \infty.
\]
This conclude the proof.
\end{proof}

\begin{proof}[Proof of \Cref{prop:weak-consistency}, (ii)]
When $\beta=0$, we have, for all $F\in \mathbb{B}$,
\[
  \mathcal{T}(F) = \int \sum_{v\in \{0,1\}^d} \frac{\langle F^*-F, \1_{A_v}\rangle}{\norm{\1_{A_v}}^2_{L^2}} \1_{A_v}(\cdot)\,Q_0(\rmd \xi)
\]
and $(\hat{F}_t)_{t\geq 0}$ solves $\deriv{t}\hat{F}_t = \mathcal{T}(\hat{F}_t)$.
It is clear that this extends to a continuous map $\mathcal{T}:L^2\to L^2$, and that
\[
  \mathcal{L} : G \mapsto -\mathcal{T}(F^*+G) \;=\; \int \sum_{v\in \{0,1\}^d} \frac{\langle G, \1_{A_v}\rangle}{\norm{\1_{A_v}}^2_{L^2}} \1_{A_v}(\cdot)\,Q_0(\rmd \xi)
\]
is a bounded positive semi-definite linear operator.
Furthermore, $G_t := \hat{F}_t-F^*$ satisfies $\deriv{t}G_t = -\mathcal{L}(G_t)$, so that $G_t=e^{-t\mathcal{L}}G_0$, with $G_0=\hat{F}_0-F^*$.

In the context of regression, notice that we can reformulate the equivalence of (i) and (ii) of \Cref{prop:critical-points} by:
\[
  \mathcal{T}(F) = 0 \iff \text{for all }A\in \mathcal{A}_d, \langle F^*-F,\1_{A}\rangle = 0.
\]
In other words, we have $\ker(\mathcal{L})=(L^2_d)^{\perp}$.
Therefore if $\mathcal{P}_d^{\perp}$ denotes the orthogonal projection on $(L_d^2)^{\perp}$,
by decomposing $G_0 = \mathcal{P}_d(G_0)+\mathcal{P}_d^{\perp}(G_0)$, we get
\[
G_t = e^{-t\mathcal{L}}\mathcal{P}_d(G_0) + \mathcal{P}_d^{\perp}(G_0).
\]
with $e^{-t\mathcal{L}}\mathcal{P}_d(G_0) \to 0$ in $L^2$, since the restriction of $\mathcal{L}$ to $L^2_d$ is positive definite.
Finally, we have
\[
\hat{F}_t = G_t+F^* \longrightarrow \mathcal{P}_d^{\perp}(G_0)+F^* \; =\; -\mathcal{P}_d^{\perp}(F^*)+F^* \;=\;\mathcal{P}_d(F^*).
\]
in $L^2$, concluding the proof.
\end{proof}

\appendix

\section{Proofs related to Sections~\ref{sec:preliminaries-w-norm} and~\ref{sec:prelim-glivenko-cantelli}}\label{sec:proofs-preliminaries}

\subsection{Proof of Lemma~\ref{lem:bound-norm-Wq} and Proposition~\ref{prop:W-to-uniform-norm}}

\begin{proof}[Proof of \Cref{lem:bound-norm-Wq}]
  First consider a measure $\nu$ of the form
  \[
  \nu(\rmd z) = \int\psi(z,\xi)\,\pi_\xi(\rmd z)\,Q_0(\rmd \xi),
  \]
  and let us show that for all $q\in [1,\infty]$,  
  \begin{equation} \label{eq:radon-nykodym-lp}
    \norm[\big]{\frac{\rmd\nu}{\rmd\pi_0}}_{L^q(\pi_0)} \;\leq\; (2^{p+d})^{1-\frac{1}{q}}\norm[\big]{\psi(z,\xi)}_{L^q(\pi_\xi(\rmd z)Q_0(\rmd \xi))}.
  \end{equation}
  It suffices to show the result for $q <\infty$, since the case $q=\infty$ is then obtained by taking the limit $q\to\infty$. Therefore, we fix $q\in [1,\infty)$ and let $q'\in (1,\infty]$ be such that $\frac{1}{q}+\frac{1}{q'}=1$.
  We use the duality between $L^q(\pi_0)$ and $L^{q'}(\pi_0)$, more precisely the fact that
  \[
  \norm[\big]{\frac{\rmd\nu}{\rmd\pi_0}}_{L^q(\pi_0)} = \sup\big\{ \int h \,\rmd\nu\ :\ h\in L^{q'}(\pi_0),\,\norm{h}_{L^{q'}(\pi_0)}= 1\big\}.
  \]
  For $h\in L^{q'}(\pi_0)$ such that $\norm{h}_{L^{q'}(\pi_0)}=1$, we compute
  \begin{align*}
    \int h\,\rmd \nu &= \int\int h(z) \psi(z,\xi)\,\pi_\xi(\rmd z)\,Q_0(\rmd \xi)\\
    &\leq \int \norm{h}_{L^{q'}(\pi_\xi)}\norm{\psi(\cdot,\xi)}_{L^q(\pi_\xi)}\,Q_0(\rmd \xi)\\
    &\leq C^{1/q'}\norm{h}_{L^{q'}(\pi_0)} \left(\int\norm{\psi(\cdot,\xi)}_{L^q(\pi_\xi)}^q\,Q_0(\rmd \xi)\right)^{\frac{1}{q}} \\
    &= C^{1/q'}\left(\int\int\abs{\psi(z,\xi)}^q \,\pi_\xi(\rmd z)\,Q_0(\rmd \xi)\right)^{\frac{1}{q}},
  \end{align*}
  where $C$ is the total mass of $\tilde{\pi}_0(\rmd z)=\int\pi_\xi(\rmd z)\,Q_0(\rmd \xi)$.
  For both inequalities above, we used Hölder's inequality --- applied to $\pi_\xi$, then to $Q_0$. By definition of $\pi_\xi$, for all splitting scheme $\xi$ of depth $d$, $\pi_\xi$ has at most total mass $2^{p+d}$, therefore $C\leq 2^{p+d}$.  This proves \eqref{eq:radon-nykodym-lp}.
  
  To prove the lemma, note that for a splitting scheme $\xi$ with associated partition $(A_v)_{v\in\{0,1\}^d}$, there exist for each $v$ a point measure $\nu_{A_v}$ supported by the corner points $c(A_v)$ such that $\1_{A_v}(z)=\nu_{A_v}([0,z])$ for all $z\in [0,1]^p$ \citep[see][Proposition~3.3]{DD21}. By definition of $\pi_\xi$, $\nu_{A_v}$ is absolutely continuous with respect to $\pi_\xi$, and $\abs{\frac{\rmd\nu_{A_v}}{\rmd \pi_\xi}(z)}\leq 1$ for all $z$ in its support. Then, if $T$ is of the form
  \[
  T(z) = \int \sum_{v\in\{0,1\}^d}\psi(v,\xi)\1_{A_v}(z)\,Q_0(\rmd \xi),
  \]
 the measure $\nu_T$ defined by
  \[
  \nu_T(\rmd z) = \int \sum_{v\in\{0,1\}^d}\psi(v,\xi)\frac{\rmd\nu_{A_v}}{\rmd \pi_\xi}(z)\,\pi_\xi(\rmd z)Q_0(\rmd \xi)
  \]
  is such that
  \[
  T(z)=\nu_T([0,z])=\int_{[0,z]} \frac{\rmd \nu_T}{\rmd\pi_0}(u)\pi_0(\rmd u). 
  \]
  Therefore, we can bound $\norm{T}_{\mathbb{W}^q} = \norm[\big]{\frac{\rmd \nu_T}{\rmd\pi_0}}_{L^q(\pi_0)}$ by
  \begin{align*}
    \norm{T}_{\mathbb{W}^q}  &\leq (2^{p+d})^{1/q'}\left(\int \abs[\Big]{\sum_{v\in\{0,1\}^d}\psi(v,\xi)\frac{\rmd\nu_{A_v}}{\rmd \pi_\xi}(z)}^q\,\pi_\xi(\rmd z)Q_0(\rmd \xi)\right)^{\frac{1}{q}}\\
    &\leq (2^{p+d})^{1/q'}\left(\int \Big(\sum_{v\in\{0,1\}^d}\abs{\psi(v,\xi)}\Big)^q\,\pi_\xi(\rmd z)Q_0(\rmd \xi)\right)^{\frac{1}{q}}\\
    &\leq (2^{p+d})^{1/q'}(2^{p+d})^{1/q}\left(\int \Big(\sum_{v\in\{0,1\}^d}\abs{\psi(v,\xi)}\Big)^q\,Q_0(\rmd \xi)\right)^{\frac{1}{q}}\\
    &\leq 2^{p+d}\sum_{v\in\{0,1\}^d}\norm{\psi(v,\cdot)}_{L^q(Q_0)},
  \end{align*}
  which concludes the proof.
\end{proof}

For the proof of \Cref{prop:W-to-uniform-norm}, the following technical lemma will be useful.
\begin{lemma} \label{lem:pi0-bound}
  Let $a<b\in [0,1]$, and consider $S_{a,b} = \{x\in [0,1]^p \ :\ a < x_1\leq b\}$.
  Then there is a constant $C'$ that depends on $d$ and $p$ such that
  \begin{equation} \label{eq:pi0-slice}
    \pi_0(S_{a,b}) \;\leq\; C' (b-a)\big(1-\log(b-a)\big)^{d-1}.
  \end{equation}
  As a consequence, if $x,y\in [0,1]^p$, we have
  \begin{equation} \label{eq:pi0-bound}
    \pi_0([0,x]\symdiff[0,y]) \;\leq\; C \norm{x-y}_{\infty}\big(1-\log \norm{x-y}_{\infty}\big)^{d-1},
  \end{equation}
  where $C = pC'$ and $\symdiff$ denotes the symmetric difference.
\end{lemma}

\begin{proof}
  Consider $\xi$ a splitting scheme of depth $d$ drawn according to $Q_0$, and let us fix $v\in \{0,1\}^d$ a leaf of the discrete binary tree.
  This leaf $v$ corresponds to a unique chain of random rectangular sets $A_v^0 = [0,1]^p \supset A_v^1 \supset \dots \supset A_v^d = A_v(\xi)$ that correspond to the subsequent splits along the branch ending in $v$.
  Let us define the following quantities:
  \begin{itemize}
    \item Let $N_v$ be the number of atoms of $\pi_\xi$ in $S_{a,b}$ that are caused by the splits along the branch $(A_v^1, \dots, A_v^d)$.
    \item For $r \leq a \leq b\leq s$ and $0\leq k\leq d-1$, let $Q_v^k(a,b,r,s)$ be the conditional probability that $N_v$ is positive, given $\Pi_1(A_v^k=[r,s\rangle$, where $\Pi_1:\mathbb{R}^p\to\mathbb{R}$ denotes the projection on the first coordinate and  $[r,s\rangle=[r,s)$ if $s < 1$ and $[r,1]$ if $s=1$.
  \end{itemize}
  We want to bound $\E[N_v]$ from above, and since the number of atoms caused by $d$ split is at most $d2^{p}$, we have
  \[
  \E[N_v] \leq d2^{p} Q_v^0(a,b,0,1).
  \]
  Now we claim that for all $r \leq a \leq b \leq s$ and $0\leq k\leq d-1$, we have
  \begin{equation} \label{eqpr:bound-Q}
    Q_v^k(a,b,r,s) \leq Q_0^k(0,b-a,0,s-r).
  \end{equation}
  The right-hand side --- note that we have $v=0$ there --- will be easily bounded further in the proof.
  We prove this by induction on the decreasing value of $k$.
  First for $k=d-1$, note that
  \[
  Q_v^{d-1}(a,b,r,s) = \frac{1}{p}\cdot\frac{b-a}{s-r}.
  \]
To see this, note that the conditioning $\Pi_1(A^{d-1}_v)=[r,s\rangle$ implies that no splits up to stage $d-1$ may have caused atoms in $S_{a,b}$, so the only possibility for an atom is for the last split to be on the first coordinate, and with splitting value between $a$ and $b$.
  This expression clearly shows that \eqref{eqpr:bound-Q} is satisfied --- actually with an equality --- for $k=d-1$.
  Let us now proceed with the induction: for $0\leq k<d-1$, we compute
  \begin{align} \label{eqpr:Q-computation}
    Q_v^k(a,b,r,s) &=\big(1-\frac{1}{p}\big)Q_v^{k+1}(a,b,r,s)+\frac{1}{p}\cdot \frac{b-a}{r-s}\\
    & \hspace{-15mm} + \frac{1}{p}\cdot\frac{1}{s-r}\bigg(\1_{v_{k+1}=1}\int_r^aQ_v^{k+1}(a,b,r',s)\,\rmd r' +\1_{v_{k+1}=0}\int_b^sQ_v^{k+1}(a,b,r,s')\,\rmd s'\bigg). \nonumber
  \end{align}
  The key to derive this is to recall that for $v_{k+1}=1$, if the split along the first coordinate arrives at $r'\geq b$, then $\Pi_1(A_v^{k+1})=[r',s\rangle$ and conditional on that, the probability that $N_v$ is positive is null.
  Now we can use the induction hypothesis: the first integral in the display above can be bounded by
  \begin{align*}
    \int_r^aQ_v^{k+1}(a,b,r',s)\,\rmd r' &\leq \int_r^aQ_v^{k+1}(0,b-a,0,s-r')\,\rmd r'\\
    & = \int_{s-a}^{s-r}Q_0^{k+1}(0,b-a,0,r')\,\rmd r'\\
    & \leq \int_{b-a}^{s-r}Q_0^{k+1}(0,b-a,0,r')\,\rmd r'.
  \end{align*}
  The second integral can be bounded by the same term using the same technique, so regardless of the value of $v_{k+1}$, we get
  \begin{align*}
    Q_v^k(a,b,r,s) &\leq \big(1-\frac{1}{p}\big)Q_0^{k+1}(0,b-a,0,s-r)+\frac{1}{p}\cdot \frac{b-a}{r-s}\\
    & \quad + \frac{1}{p}\cdot\frac{1}{s-r}\int_{b-a}^{s-r}Q_0^{k+1}(0,b-a,0,s')\,\rmd s'\\
    & = Q_0^k(0,b-a,0,s-r).
  \end{align*}
  To get the equality, we simply applied \eqref{eqpr:Q-computation} to $Q_0^k(0,b-a,0,s-r)$.
  So the induction is proven and \eqref{eqpr:bound-Q} follows. In particular,
  \[
  \E[N_v] \leq d2^{p} Q_v^0(a,b,0,1) \leq d2^{p} Q_0^0(0,b-a,0,1).
  \]
  This last quantity is finally bounded from above: in the worst case scenario, we split only the first coordinate, so after $k$ split $A_0^k$ is of the form $[0,U_1\cdots U_k)$, where the $U_i$ are i.i.d.\ uniform on $[0,1]$.
  Therefore, we have 
  \begin{align*}
    Q_0^0(0,b-a,0,1) &\leq \P(U_1\cdots U_d \leq b-a)\\
    & = \P(\Gamma_d \geq -\log (b-a)) \\
    & = (b-a)\sum_{k=0}^{d-1}\frac{(-\log(b-a))^{k}}{k!},
  \end{align*}
  where $\Gamma_d$ follows a $\Gamma(d,1)$ distribution.
  It follows  that there is a constant $C''$ that depends only on $d$ such that $Q_0^0(0,b-a,0,1) \leq C''(b-a)(-\log(b-a))^{d-1}$.
  \Cref{eq:pi0-slice} now follows  from
  \[
  \pi_0(S_{a,b}) \leq \sum_{v}\E[N_v] \leq d2^{p+d} Q_0^0(0,b-a,0,1).
  \]
  To show \eqref{eq:pi0-bound}, consider $x,y\in [0,1]^p$ and fix some $a_i,b_i\in [0,1]$ for each $i\in [\![1,p]\!]$, such that
  \[
  a_i \leq x_i\wedge y_i\leq x_i\vee y_i\leq b_i = a_i+\norm{x-y}_{\infty}.
  \]
  Since $[0,x]\symdiff[0,y] \subset \bigcup_{i}S_{a_i,b_i}$, we apply \eqref{eq:pi0-slice} to each of the $S_{a_i,b_i}$, which completes the proof.
\end{proof}

We can now prove \Cref{prop:W-to-uniform-norm}.

\begin{proof}[Proof of \Cref{prop:W-to-uniform-norm}]
  For any $x,y\in [0,1]^p$, using Hölder's inequality, we have
  \begin{align*}
  \abs{F(x)-F(y)} &= \abs[\bigg]{\int(\1_{[0,x]}(z)-\1_{[0,y]}(z))\frac{\rmd\nu_F}{\rmd\pi_0}\,\rmd \pi_0}\\
  &\leq \big( \pi_0([0,x]\symdiff[0,y])\big)^{\frac{1}{q'}} \bigg(\int\abs[\Big]{\frac{\rmd\nu_F}{\rmd\pi_0}}^q\,\rmd \pi_0\bigg)^{\frac{1}{q}},\\
  &\leq \big( \pi_0([0,x]\symdiff[0,y])\big)^{\frac{1}{q'}} \norm{F}_{\mathbb{W}^{(q)}}.
  \end{align*}
The bound for $\pi_0([0,x]\symdiff[0,y])$ is obtained using \Cref{lem:pi0-bound}.
\end{proof}

\subsection{Proof of Proposition~\ref{prop:GC}}

The following lemma will be useful for the proof of \Cref{prop:GC}.
Before stating and proving it, let us define the notions of brackets and envelope functions, as they are used in e.g.\ \citet[Section 2.4]{vW96}.

Consider $\mathcal{X}$ a separable complete metric space endowed with its Borel $\sigma$-field, and let $\mathbb{P}$ denote a probability measure on $\mathcal{X}$.
For any $\epsilon>0$, an $\epsilon$-bracketing of a family $\mathscr{F}$ of measurable functions $f:X\to \mathbb{R}$, is a collection of pairs $(l_i,u_i)_{i\in I}$ of $\mathbb{P}$-integrable functions $X\to\mathbb{R}$ satisfying
\[
\forall i\in I, \int_{\mathcal{X}} (u_i-l_i)\,\rmd \mathbb{P} \leq \epsilon\quad \text{ and }\quad \bigcup_{i\in I} \{f\in \mathscr{F}\ :\ l_i\leq f \leq u_i\} = \mathscr{F}.
\]
If for all $\epsilon>0$, such an $\epsilon$-bracketing can be found with a finite index set $I$, we say that $\mathscr{F}$ has finite \emph{bracketing numbers} for $\mathbb{P}$, and we write
\[
  N_\epsilon(\mathscr{F}) = \inf \big\{\abs{I}\ :\ (l_i,u_i)_{i\in I}\; \text{is an }\epsilon\text{-bracketing of }\mathscr{F}\big\}.
\]
A measurable function $F:X\to \mathbb{R}$ is called an envelope function of $\mathscr{F}$ if
\[
  \sup_{f\in \mathscr{F}} \abs{f} \leq F \qquad \mathbb{P}\text{-a.s.}
\]
Finally, recall that $\mathcal{A}$ denotes the class of hyperrectangles, of the form $[a,b]\subset [0,1]^p$.

\begin{lemma} \label{lem:product-GC}
  If $\mathscr{G}$ is a class of measurable functions $g:[0,1]^p\times \mathbb{R} \to \mathbb{R}$ with finite bracketing numbers for $\P$ and with a $\P$-integrable envelope function $G$, then the class of functions
  \[
    \mathscr{G}\cdot \mathcal{A} := \{(x,y)\mapsto g(x,y)\1_A(x)\ :\ g\in\mathscr{G},A\in\mathcal{A}\}
  \]
  has finite bracketing numbers and therefore is $\P$-Glivenko--Cantelli.
  Furthermore, for a fixed $\epsilon>0$, if $\delta(\epsilon)>0$ denotes a value such that for each Borel set $B\subset [0,1]^p$,
  \[
    \P(X\in B) \leq \delta(\epsilon)\quad\mbox{implies}\quad \P[G(X,Y)\1_{B}(X)] \leq \epsilon/3,
  \]
  then $N_\epsilon(\mathscr{G}\cdot \mathcal{A}) = O\big(N_{\epsilon/3}(\mathscr{G}) \delta(\epsilon)^{-2p}\big)$.
\end{lemma}
\begin{proof}
  It is classical that the class of indicator functions of rectangles has finite bracketing numbers for any probability measure on $[0,1]^p$, but we nevertheless recall the argument in order to fix some notation.
  For all $\epsilon>0$, for each coordinate $j\in[\![1,p]\!]$, we can define a sequence $0=a^j_0 < a^j_1 < \dots < a^j_{k_j}$ with $k_j \leq 2p/\epsilon+1$ such that
  \[
  \mathbb{P}(X^i \in (a^j_l, a^j_{l+1})) \leq \frac{\epsilon}{2p}.
  \]
  Now consider the subset $\mathcal{A}_\epsilon\subset\mathcal{A}$ of rectangles $A=[a,b]$ such that for each $j$, the $j$th coordinate of $a$ and $b$ is among the $(a^j_l,0\leq l \leq k_j)$.
  Note that the cardinality of $\mathcal{A}_\epsilon$ is no greater than $\prod_j \frac{k_j(k_j+1)}{2} = O(\epsilon^{-2p})$.
  For each rectangle $A\in \mathcal{A}$, let us define
  \[
    \overline{A} = \bigcap\big\{A'\in \mathcal{A}_\epsilon\ :\ A\subset A' \big\},
  \]
  and define $\underline{A}$ as the largest rectangle of $\mathcal{A}_\epsilon$ included in $\overline{A}$ whose boundary is disjoint from that of $\overline{A}$, with the convention $\underline{A}=\emptyset$ if there are no such rectangles in $\mathcal{A}_\epsilon$.
  Then it is clear that
  \[
  \underline{A}\subset A \subset \overline{A}\quad\text{and}\quad \P(X\in \overline{A}\setminus\underline{A})\leq \epsilon.
  \]
  In the rest of the proof, with a slight abuse, we call the pair $(\underline{A},\overline{A})$ an $\epsilon$-bracketing of~$A$.
  
  For all $\epsilon>0$, let us choose $\mathscr{G}_\epsilon$ a finite $\epsilon$-bracketing of $\mathscr{G}$.
  We fix $\epsilon>0$ and, using the fact that $G$ is $\P$-integrable, we define $\delta = \delta(\epsilon)$ as in the statement of the lemma.
  Consider $g\in \mathscr{G}$ and $A\in \mathcal{A}$.
  Let $(\underline{g},\overline{g})\in \mathscr{G}_{\epsilon/3}$ be such that $\underline{g}\leq g\leq \overline{g}$, and $(\underline{A},\overline{A})$ be a $\delta$-bracketing of $A$.
  Then we have
  \[
    \underline{g}(x,y)\1_{\underline{A}}(x) - G(x,y)\1_{\overline{A}\setminus\underline{A}}(x) \leq g(x,y)\1_A(x) \leq \overline{g}(x,y)\1_{\underline{A}}(x) + G(x,y)\1_{\overline{A}\setminus\underline{A}}(x),
  \]
  and
  \begin{align*}
  &\P\Big[\overline{g}(X,Y)\1_{\underline{A}}(X) + G(X,Y)\1_{\overline{A}\setminus\underline{A}}(X) - \big(\underline{g}(X,Y)\1_{\underline{A}}(X) - G(X,Y)\1_{\overline{A}\setminus\underline{A}}(X)\big)\Big]\\
  &= \P\Big[\big(\overline{g}(X,Y)-\underline{g}(X,Y) \big)\1_{\underline{A}}(X) + 2 G(X,Y)\1_{\overline{A}\setminus\underline{A}}(X)\Big]\\
  & \leq \frac{\epsilon}{3} + 2\frac{\epsilon}{3} \;=\; \epsilon.
  \end{align*}
  Therefore, we have found a class of functions that bracket $\mathscr{G}\cdot\mathcal{A}$ with an $L^1(\P)$-precision $\epsilon$, and there are at most $\abs{\mathscr{G}_{\epsilon/3}}\abs{\mathcal{A}_{\delta}}$ such functions.
  
  It remains to argue that $\mathscr{G}\cdot \mathcal{A}$ is $\P$-Glivenko--Cantelli: this is easily deduced from the fact that it has finite bracketing numbers --- see \citet[Theorem 2.4.1]{vW96}.
\end{proof}

\begin{proof}[Proof of \Cref{prop:GC}.]  
  We aim to apply \Cref{lem:product-GC} twice.
  Indeed, note that with the \namecref{lem:product-GC}'s notation, the sets $\mathcal{F}_1$ and $\mathcal{F}_2$ are of the form
  \[
    \mathcal{F}_i = \mathcal{F}'_i\cdot \mathcal{A},
  \]
  with
  \begin{gather*}
    \mathcal{F}'_1 = \left\{(x,y)\mapsto \pderiv{z}L(y,F(x))\ :\ F\in B
    \right\}\\
    \mathcal{F}'_2=\left\{(x,y)\mapsto \pderiv[2]{z}L(y,F(x))\ :\ F\in B\right\}.
  \end{gather*}
  It then suffices to show that the bracketing numbers of $\mathcal{F}'_1$ and $\mathcal{F}'_2$ are finite, and that these classes of functions have a $\P$-integrable envelope.
  Since the proof is the same for the two classes of functions, we show this only for $\mathcal{F}'_1$.
  
  Let us define $M=\sup_{F\in B}\norm{F}_{\mathbb{W}^q}$, and $q'=q/(q-1)$.
  Note that for all $F\in B$, we have $\norm{F}_\infty \leq \norm{F}_{\mathbb{W}^q} \leq M$, and by \Cref{prop:W-to-uniform-norm} all $F\in B$ have a common modulus of continuity
  \begin{equation}\label{eqproof:bound-modulus-continuity}
  \omega(\delta) = \sup_{F\in B}\sup_{\substack{x,y\in [0,1]^p\\\norm{x-y}_\infty \leq \delta}}\abs{F(x)-F(y)} \leq \left(C\delta(1-\log \delta)^{d-1}\right)^{\frac{1}{q'}}M,
  \end{equation}
  where $C$ is a constant that depends only on $d$ and $p$.
  Note that since $\sup_{F\in B}\norm{F}_{\infty}\leq M$, then $G(x,y) := \sup_{z\in [-M,M]}\abs[\big]{\pderiv{z}L(y,z)}$ is a $\P$-integrable (by \Cref{ass:Y-integrable}) envelope function for $\mathcal{F}'_1$.
  Let $m=m(\epsilon)>0$ be such that 
  \[
    \mathbb{E}\left[G(X,Y)\1_{\abs{Y}\geq m}\right] \leq \frac{\epsilon}{2}.
  \]
  Since by \Cref{ass:L-regular-convex} $\pderiv{z}{L}(y,z)$ and $\pderiv[2]{z}{L}(y,z)$ are locally Lipschitz, the functions in $\mathcal{F}'_1$ are uniformly bounded on $[0,1]^p\times[-m,m]$ and, on this compact space, have a common modulus of continuity satisfying the same bound as in \eqref{eqproof:bound-modulus-continuity} up to a multiplicative constant.
  It is classical \citep[see][Theorem 2.7.1]{vW96} that for all $\epsilon > 0$, there exists a finite \emph{uniform} $\epsilon$-bracketing of $\mathcal{F}'_1$ on $[0,1]^p\times[-m,m]$, i.e.\ there exists a finite class of functions $\mathcal{F}'_1(\epsilon)$ such that for each $g\in \mathcal{F}'_1$, there exist $\underline{g},\overline{g}\in \mathcal{F}'_1(\epsilon)$ satisfying
  \[
    \underline{g}(x,y) \leq g(x,y) \leq \overline{g}(x,y) \leq \underline{g}(x,y)+\epsilon, \qquad (x,y)\in [0,1]^p\times[-m,m].
  \]
  Now it is readily checked that the functions of the form
  \[
    \tilde{g}(x,y) = \begin{cases}
    g_0(x,y) & \text{if }(x,y)\in [0,1]^p\times [-m,m]\\
    G(x,y) & \text{otherwise}
    \end{cases} \qquad \text{for some }g_0\in \mathcal{F}'_1(\epsilon)
  \]
  or 
  \[
  \tilde{g}(x,y) = \begin{cases}
  g_0(x,y) & \text{if }(x,y)\in [0,1]^p\times [-m,m]\\
  -G(x,y) & \text{otherwise}
  \end{cases} \qquad \text{for some }g_0\in \mathcal{F}'_1(\epsilon)
  \]
  define a finite class of functions such that for each $g\in \mathcal{F}'_1$, there exists  $\underline{g},\overline{g}$ of the previous form satisfying
  \begin{gather*}
  \underline{g}(x,y) \leq g(x,y) \leq \overline{g}(x,y), \qquad (x,y)\in [0,1]^p\times\mathbb{R}\\
  \text{and}\qquad \E\big[\overline{g}(X,Y)-\underline{g}(X,Y)\big] \leq 2\epsilon.
  \end{gather*}
  This concludes the proof.
\end{proof}

\phantomsection

\end{document}